%% file: main_neurips.tex
\documentclass[11pt]{article}
\usepackage[margin=1.2in]{geometry}

\usepackage{dor_macros}

\usepackage{xcolor,colortbl}

\usepackage[utf8]{inputenc} 
\usepackage[T1]{fontenc}    
\usepackage{hyperref}       
\usepackage{url}            
\usepackage{booktabs}       
\usepackage{amsfonts}       
\usepackage{nicefrac}       
\usepackage{microtype}      
\usepackage{xcolor}         

\title{Max-Sliced Mutual Information}

\author{%
Dor Tsur\thanks{\texttt{dortz@post.bgu.ac.il}}\\
Ben-Gurion University
\and
Ziv Goldfeld\thanks{\texttt{goldfeld@cornell.edu}}\\
Cornell University
\and 
Kristjan Greenewald\thanks{\texttt{kristjan.h.greenewald@ibm.com}} \\
MIT-IBM Watson AI Lab
}
  
\date{}

\begin{document}

\maketitle

\begin{abstract}
Quantifying the dependence between high-dimensional random variables is central to statistical learning and inference. Two classical methods are canonical correlation analysis (CCA), which identifies maximally correlated projected versions of the original variables, and Shannon's mutual information, which is a universal dependence measure that also captures high-order dependencies. However, CCA only accounts for linear dependence, which may be insufficient for certain applications, while mutual information is often infeasible to compute/estimate in high dimensions. This work proposes a middle ground in the form of a scalable information-theoretic generalization of CCA, termed max-sliced mutual information (mSMI). mSMI equals the maximal mutual information between low-dimensional projections of the high-dimensional variables, which reduces back to CCA in the Gaussian case. It enjoys the best of both worlds: capturing intricate dependencies in the data while being amenable to fast computation and scalable estimation from samples. We show that mSMI retains favorable structural properties of Shannon's mutual information, like variational forms and identification of independence. We then study statistical estimation of mSMI, propose an efficiently computable neural estimator, and couple it with formal non-asymptotic error bounds. We present experiments that demonstrate the utility of mSMI for several tasks, encompassing independence testing, multi-view representation learning, algorithmic fairness, and generative modeling. We observe that mSMI consistently outperforms competing methods with little-to-no computational overhead.  
\end{abstract}

\section{Introduction}
Dependence measures between random variables are fundamental in statistics and machine learning for tasks spanning independence testing \cite{siegel1957nonparametric,haugh1976checking,berrett2019nonparametric}, clustering \cite{butte1999mutual,kraskov2005hierarchical}, representation learning \cite{oord2018representation,wang2015deep}, and self-supervised learning \cite{chen2020simple,zbontar2021barlow,balestriero2023cookbook}. 
There are a wide array of measures quantifying different notions of dependence, with varying statistical and computational complexities. The simplest is the Pearson correlation coefficient \cite{pearson1895vii}, which only captures linear dependencies. At the other extreme is Shannon's mutual information \cite{CovThom06,el2011network}, which is a universal dependence measure that is able to identify arbitrarily intricate dependence structures. Despite its universality and favorable properties, accurately estimating mutual information from data is infeasible in high-dimensional settings. First, mutual information estimation rates suffers from the curse of dimensionality, whereby convergence rates deteriorate exponentially with dimension~\cite{Paninski2003}. Additionally, computing mutual information  requires integrating log-likelihood ratios over a high-dimensional ambient space, which is generally intractable.

Between these two extremes is the popular canonical correlation analysis (CCA)~\cite{hotelling1933analysis}, which identifies maximally correlated linear projections of variables. However, classical CCA still only captures linear dependence, which has inspired nonlinear extensions such as Hirschfeld–Gebelein–R\'enyi (HGR) maximum correlation \cite{hirschfeld1935connection,gebelein1941statistische,renyi1959measures}, kernel CCA~\cite{akaho2006kernel, hardoon2004canonical}, deep CCA~\cite{andrew2013deep, wang2015deep}, and various other  generalizations~\cite{bach2005probabilistic,kim2006learning,parkhomenko2009sparse,painsky2020nonlinear}. However, HGR is computationally infeasible, while kernel and deep CCA can be burdensome in high dimensions as they require optimization over reproducing kernel Hilbert spaces or deep neural networks (NNs), respectively. To overcome these shortcomings, this work proposes max-sliced mutual information (mSMI)---a scalable information-theoretic extension of CCA that captures the full dependence structure while only requiring optimization over linear projections.

\subsection{Contributions}

The mSMI is defined as the maximal mutual information between linear projections of the variables. Namely, the $k$-dimensional mSMI between $X$ and $Y$ with values in $\RR^{d_x}$ and $\RR^{d_y}$, respectively, is\footnote{The parameter $k$ is fixed and small compared to the ambient dimensions $d_x,d_y$, often simply set as $k=1$.}
\begin{equation}    \msmi_k(X;Y):=\sup_{(\rA,\rB)\in\sti(k,d_x)\times\sti(k,d_y)}\sI(\rAt X;\rBt Y),\label{eq:msmi_intro}
\end{equation}
where $\sti(k,d)$ is the Stiefel manifold of $d \times k$ matrices with orthonormal columns. Unlike the nonlinear CCA variants that use nonlinear feature extractors in the high-dimensional ambient space, mSMI retains the linear projections of CCA and captures nonlinear structures in the \emph{low-dimensional} feature space. This is done by using the mutual information between the projected variables, rather than the correlation, as the optimization  objective. Beyond being considerably simpler from a computational standpoint, this crucial difference allows mSMI to identify the full dependence structure, akin to classical mutual information. 
mSMI can also be viewed as the maximized version of the average-sliced mutual information (aSMI) \cite{goldfeld2021sliced,goldfeld2022k}, which averages $\sI(\rAt X;\rBt Y)$ with respect to (w.r.t.) the Haar measure over $\sti(k,d_x)\times\sti(k,d_y)$.
However, we demonstrate that compared to aSMI, mSMI benefits from improved neural estimation error bounds and a clearer interpretation.

We show that mSMI inherits important properties of mutual information, including identification of independence, tensorization, and variational forms. For jointly Gaussian $(X,Y)$, the optimal mSMI projections coincide with those of $k$-dimensional CCA \cite{mardia1979multivariate}, posing mSMI as a natural information-theoretic generalization. Beyond the Gaussian case, the solutions differ and mSMI may yield more effective representations for downstream tasks due to the intricate dependencies captured by mutual information. We demonstrate this superiority empirically for multi-view representation learning. 


For efficient computation, we propose an mSMI neural estimator based on the Donsker-Varadhan (DV) variational form \cite{donsker1983asymptotic}. Neural estimators have seen a surge in interest due to their scalability and compatibility with gradient-based optimization \cite{belghazi2018mutual,poole2018variational,chan2019neural,song2019understanding,zhang2019itene,mukherjee2020ccmi,tsur2023neural,guo2022tight}. Our estimator employs a single model that composes the projections with the NN proxy of the DV critic and jointly optimizes them. 
This results in both the estimated mSMI value and the optimal projection matrices. Building on recent analysis of neural estimation of $f$-divergences \cite{sreekumar2021non,sreekumar2022neural}, we establish non-asymptotic error bounds that scale as $O\big(k^{1/2}(\ell^{-1/2}+ kn^{-1/2})\big)$, where $\ell$ and $n$ are the numbers of neurons and $(X,Y)$ samples, respectively. Equating $\ell$ and $n$ results in the (minimax optimal) parametric estimation rate, which highlights the scalability of mSMI and its compatibility to modern learning settings. 


In our empirical investigation, we first demonstrate that our mSMI neural estimator converges orders of magnitude faster than that of aSMI \cite{goldfeld2022k}. This is because the latter requires (parallel) training of many neural estimators corresponding to different projection directions, while the mSMI estimator optimizes a single combined model. Notwithstanding the reduction in computational overhead, we show that mSMI outperforms average-slicing for independence testing. Next, we compare mSMI with deep CCA \cite{andrew2013deep,wang2015deep} by examining downstream classification accuracy based on representations obtained from both methods in a multi-view learning setting. Remarkably, we observe that even the linear mSMI projections outperform nonlinear representations obtained from deep CCA. We also consider an application to algorithmic fairness under the infomin framework~\cite{chen2023scalable}. Replacing their generalized Pearson correlation objective with mSMI, we again observe superior performance in the form of more fair representations whose utility remains on par with the fairness-agnostic model. Lastly, we devise a max-sliced version of the InfoGAN by replacing the classic mutual information regularizer with its max-sliced analog. We show that despite the low-dimensional projections, the max-sliced InfoGAN successfully learns to disentangle the latent space and generates quality samples. 

\section{Background and Preliminaries}\label{sec:background}
\textbf{Notation.} For $a,b\in\RR$, we use the notation $a \wedge b = \min \{ a,b \}$ and $a \vee b = \max \{a,b \}$. For $d\geq 1$, $\| \cdot \|$ is the Euclidean norm in $\RR^d$. 
The Stiefel manifold of $d \times k$ matrices with orthonormal columns is denoted by $\sti(k,d)$. For a $d \times k$ matrix $\rA$, we use $\proj^\rA:\RR^d\to\RR^k$ for the orthogonal projection onto the row space of $\rA$. For $\rA\in\RR^{d\times k}$ with $\mathrm{rank}(\rA)=r\leq k\wedge d$, we write $\sigma_1(\rA),\ldots,\sigma_r(\rA)$ for its non-zero singular values, and assume without loss of generality (w.l.o.g.) that they are arranged in descending order. Similarly, the eigenvalues of a square matrix $\Sigma\in\RR^{d\times d}$ are denoted by $\lambda_1(\Sigma),\ldots,\lambda_d(\Sigma)$. Let $\calP(\R^d)$ denote the space of Borel probability measures on $\RR^d$. For $\mu,\nu\in\cP(\RR^d)$, we use $\mu\otimes\nu$ to denote a product measure, while $\supp(\mu)$ designates the support of $\mu$. All random variables  throughout are assumed to be continuous w.r.t. the Lebesgue measure. For a  measurable map $f$, the pushforward of $\mu$ under $f$ is denoted by $f_{\sharp}\mu = \mu \circ f^{-1}$, i.e., if $X \sim \mu$ then $f(X) \sim f_{\sharp}\mu$. For a jointly distributed pair $(X,Y)\sim \mu_{XY}\in\cP(\RR^{d_x}\times\RR^{d_y})$, we write $\Sigma_X$ and $\Sigma_{XY}$ for covariance matrix of $X$ and cross-covariance matrix of $(X,Y)$, respectively.

\paragraph{Canonical correlation analysis.}
CCA is a classical method for devising maximally correlated linear projections of a pair of random variables $(X,Y)\sim\mu_{XY}\in\cP(\RR^{d_x}\times\RR^{d_y})$ via~\cite{hotelling1933analysis} 
\begin{equation}\label{eq:cca_nonconvex}
(\theta_\cca,\phi_\cca)=\argmax_{(\phi,\theta)\in\RR^{d_x}\times\RR^{d_y}}\frac{\thetat \Sigma_{XY}\phit}{\sqrt{\thetat \Sigma_{XX}\theta\phit\Sigma_{YY}\phi}}=\argmax_{\substack{(\theta,\phi)\in\RR^{d_x}\times\RR^{d_y}:\\
 \thetat \Sigma_{X}\theta=\phit\Sigma_{Y}\phi=1}} \thetat \Sigma_{XY}\phi,
\end{equation}
where the former objective is the correlation coefficient $\rho(\thetat X,\phit Y)$ between the projected variables and the equality follows from invariance of $\rho$ to scaling. 
The global optimum has an analytic form as $(\theta_{\mathsf{CCA}},\phi_{\mathsf{CCA}})=(\Sigma_X^{-1/2}\theta_1,\Sigma_Y^{-1/2}\phi_1)$, where $(\theta_1,\phi_1)$ is the (unit-length) top left and right singular vector pair associated the largest singular value of $\rT_{XY}\coloneqq\Sigma_X^{-1/2}\Sigma_{XY}\Sigma_Y^{-1/2}\in\RR^{d_x\times d_y}$. This solution is efficiently computable with $O((d_x\vee d_y)^3)$ complexity, given that the population correlation matrices are available. CCA extends to $k$-dimensional projections via the optimization \cite{mardia1979multivariate}
\begin{equation}\label{eq:k_dim_cca}
\max_{\substack{(\rA,\rB)\in\RR^{d_x\times k}\times\RR^{d_y \times k}:\\   \rA^\intercal\Sigma_X\rA=\rB^\intercal\Sigma_Y\rB=\mathrm{I}_k}} \mathrm{tr}(\rA^\intercal \Sigma_{XY}\rB),
\end{equation}
with the optimal CCA matrices being $(\rA_{\mathsf{CCA}},\rB_{\mathsf{CCA}})=(\Sigma_X^{-1/2}\rU_k,\Sigma_Y^{-1/2}\rV_k)$, where $\rU_k$ and $\rV_k$ are the matrices of the first $k$ left- and right-singular vectors of $\rT_{XY}$. Then the optimal objective value is the sum of the top $k$ singular values of $\rT_{XY}$ (the Ky Fan $k$-norm of $\rT_{XY}$).

\paragraph{Divergences and information measures.} 
Let $\mu,\nu\in\cP(\RR^d)$ satisfy $\mu\ll\nu$, i.e., $\mu$ is absolutely continuous w.r.t. $\nu$. 
The Kullback-Leibler (KL) divergence is defined as $\dkl(\mu\|\nu)\coloneqq\int_{\RR^d} \log(d\mu/d\nu)d\mu$. We have $\dkl(\mu\|\nu)\geq 0$, with equality if and only if (iff) $\mu=\nu$.
Mutual information and differential entropy are defined from the KL divergence as follows. Let $(X,Y)\sim\mu_{XY}\in\cP(\RR^{d_x}\times\RR^{d_y})$ and denote the corresponding marginal distributions by $\mu_X$ and $\mu_Y$. The mutual information between $X$ and $Y$ is given by $\sI(X;Y)\coloneqq\dkl(\mu_{XY}\|\mu_X\otimes\mu_Y)$ and serves as a measure of dependence between those random variables. The differential entropy of $X$ is defined as $\sh(X)=\sh(\mu_X)\coloneqq-\dkl(\mu_X\|\leb)$. Mutual information between (jointly) continuous variables and differential entropy are related via $\sI(X;Y)=\sh(X)+\sh(Y)-\sh(X,Y)$; decompositions in terms of conditional entropies are also available~\cite{CovThom06}. 

\section{Max-Sliced Mutual Information}
We now define the $k$-dimensional mSMI, establish structural properties thereof, and explore the Gaussian setting and its connections to CCA. We focus here on the case of (linear) $k$-dimensional projections and discuss extensions to nonlinear slicing in Section \ref{subsec:nonlin_msi}. 
\begin{definition}[Max-sliced mutual information]\label{def:msmi_def}
For $1\leq k\leq d_x\wedge d_y$, the $k$-dimensional mSMI between $(X,Y)\sim \mu_{XY}\in \cP(\RR^{d_x}\times\RR^{d_y})$ is $\msmi_k(X;Y)\coloneqq\sup_{(\rA,\rB)\in\sti(k,d_x)\times\sti(k,d_y)}\sI(\rA^\intercal X;\rB^\intercal Y)$, where $\sti(k,d)$ is the Stiefel manifold of $d \times k$ matrices with orthonormal columns.
\end{definition}
The mSMI measures Shannon's mutual information between the most informative $k$-dimensional projections of $X$ and $Y$. 
It can be viewed as a maximized version of the aSMI $\asmi_k(X;Y)$ from \cite{goldfeld2021sliced,goldfeld2022k}, defined as the integral of $\sI(\rA^\intercal X;\rB^\intercal Y)$ w.r.t. the Haar measure over $\sti(k,d_x)\times\sti(k,d_y)$. 
For $d=d_x=d_y$, we have $\asmi_d(X;Y)=\msmi_d(X;Y)=\sI(X;Y)$ due to invariance of mutual information to bijections. The supremum in mSMI is achieved since the Stiefel manifold is compact and the function $(\rA,\rB)\mapsto \sI(\rA^\intercal X;\rB^\intercal Y)$ is Lipschitz and thus continuous (Lemma 2 of \cite{goldfeld2022k}). 

\begin{remark}[Multivariate and conditional mSMI]
It is natural to extend the mSMI definition above to the multivariate and conditional cases. Let $(X,Y,Z)\sim \mu_{XYZ}\in\cP(\RR^{d_x}\times\RR^{d_y}\times \RR^{d_z})$. The $k$-dimensional multivariate and conditional mSMI functionals are, respectively, $\msmi_k(X,Y;Z)\coloneqq\max_{\rA,\rB,\rC}\sI(\rA^\intercal X, \rB^\intercal Y; \rC^\intercal Z)$ and $\msmi_k(X;Y|Z)\coloneqq\max_{\rA,\rB,\rC}\sI(\rA^\intercal X; \rB^\intercal Y| \rC^\intercal Z)$. Connections between $\msmi_k(X;Y)$ and its multivariate and conditional versions are given in the proposition to follow. We also note that one may generalize the definition of $\msmi_k(X;Y)$ to allow for projections into feature spaces of different dimensions, i.e., $\rA\in \sti(k_x,d_x)$ and $\rB\in\sti(k_y,d_y)$, for $k_x\neq k_y$. We expect our theory to extend to that case, but leave further exploration for future work.
\end{remark}

\begin{remark}[Max-sliced entropy]
In the spirit of mSMI, we define the $k$-dimensional max-sliced (differential) entropy of $X\sim \mu_X\in\cP(\RR^{d})$ as $\msh_k(X)\coloneqq\msh_k(\mu) :=\sup_{\rA\in\sti(k,d)}\sh(\rA^\intercal X)$.
An important property of classical differential entropy is the maximum entropy principle \cite{CovThom06}, which finds the highest entropy distribution within given class. In Appendix \ref{supp:msh_maximization}, we study the max-sliced entropy maximizing distribution in several common scenarios. For instance, we show that $\msh_k$ is maximized by the Gaussian distribution under a fixed (mean and) covariance constraint. Namely, letting $\cP_1(m,\Sigma):=\big\{\mu\in\cP(\RR^d):\,\supp(\mu)=\RR^d\,,\,\EE_\mu[X]=m\,,\,\EE_\mu\big[(X-m)(X-m)^\intercal \big]=\Sigma\big\}$, we have $\argmax_{\mu\in\cP_1(\mu,\Sigma)}\msh_k(\mu) = \mathcal{N}(m,\Sigma)$. 
\end{remark}

\begin{remark}[Sliced divergences]
The slicing technique has originated as a means to address scalability issues concerning statistical divergences. Significant attention was devoted to sliced Wasserstein distances as discrepancy measures between probability distributions \cite{rabin2012wasserstein, deshpande2019max, kolouri2019generalized,lin2020projection,huang2021riemannian,lin2021projection, nietert2022statistical}. As such, the sliced Wasserstein distance differs from mutual information and its sliced variants, which  quantify dependence between random variables, rather than discrepancy per se. Additionally, as Wasserstein distances are rooted in optimal transport theory, they heavily depend on the geometry of the underlying data space. Mutual information, on the other hand, is induced by the KL divergence, which only depends on the log-likelihood of the considered distributions and overlooks the geometric aspects.
\end{remark}
\subsection{Structural Properties}
The following proposition lists useful properties of the mSMI, which are similar to those of the average-sliced variant (cf. \cite[Proposition 1]{goldfeld2022k}) as well as Shannon's mutual information itself.
\begin{proposition}[Structural properties]\label{prop:msi_properties}
The following properties hold:
\begin{enumerate}[leftmargin=.45cm]
    \vspace{-2mm}
    \item \textbf{Bounds:} For any integers $k_1<k_2$: $\asmi_{k_1}(X;Y)\leq\msmi_{k_1}(X;Y) \leq\msmi_{k_2}(X;Y)\leq \s{I}(X;Y)$.
    \vspace{-1mm}
    \item \textbf{Identification of independence:} $\msmi_k(X;Y)\geq 0$ with equality iff  $(X,Y)$ are independent.
    \item \textbf{KL divergence representation:} We have
    \[\msmi_k(X;Y)=\sup_{(\rA ,\rB)\in\sti(k,d_x)\times \sti(k,d_y)} \dkl\big((\mathfrak{p}^\rA, \mathfrak{p}^\rB)_{\#}\mu_{XY}\big\|(\mathfrak{p}^\rA,\mathfrak{p}^\rB)_{\#}\mu_X\otimes\mu_Y\big),\]
    
    \vspace{-3mm}
    \item \textbf{Sub-chain rule:} For any random variables $X_1,\dots,X_n,Y$, we have
    \vspace{-1.5mm}
    \[
    \msmi_k(X_1,\ldots,X_n;Y)\leq \msmi_k(X_1;Y)+\sum_{i=2}^n\msmi_k(X_i;Y|X_1,\ldots,X_{i-1}).
    \] 
    
    \vspace{-6mm}
    \item \textbf{Tensorization:} For mutually independent $\{(X_i,\mspace{-2mu}Y_i)\}_{i=1}^n$,
    $\msmi_k\big(\mspace{-1mu}\{X_i\}_{i=1}^n;\mspace{-2mu}\{Y_i\}_{i=1}^n\mspace{-1mu}\big)\mspace{-3mu}=\mspace{-3.5mu} \sum\limits_{i=1}^n\mspace{-2mu} \msmi_k(\mspace{-1mu}X_i;\mspace{-2mu} Y_i)$.
\end{enumerate}
\end{proposition}

The proof follows by similar arguments to those in the average-sliced case, but is given for completeness in Supplement \ref{proof:msi_properties}. Of particular importance are Properties 2 and 3. The former renders mSMI sufficient for independence testing despite being significantly less complex than the classical mutual information between the high-dimensional variables. The latter, which represent mSMI as a supremized KL divergence, is the basis for neural estimation techniques explored in Section \ref{sec:neural_est}.

\begin{remark}[Relation to average-SMI]\label{rem:msmi_vs_asmi}
   Beyond the inequality relationship in Property 1 above, Proposition~4 in \cite{goldfeld2021sliced} (paraphrased) shows that for matrices $\rW_x,\rW_y$ and vectors $b_x,b_y$ or appropriate dimensions we have $\sup_{\rW_x,\rW_y,b_x,b_y}\asmi_1(\rW_x^\intercal X+b_x;\rW_y^\intercal Y+b_y)=\msmi_1(X;Y)$, and the relation readily extends to projection dimension $k>1$. In words, optimizing the aSMI over linear transformations of the high-dimensional data vectors coincides with the max-sliced version. This further justifies the interpretation of $\msmi_k(X;Y)$ as the information between the two most informative representations of $X,Y$ in a $k$-dimensional feature space. It also suggests that mSMI is compatible for feature extraction tasks, as explored in Section \ref{sec:exp_representation} ahead.
\end{remark}

\subsection{Gaussian Max-SMI versus CCA}\label{subsec:Gaussian_msi}

The mSMI is an information-theoretic extension of the CCA coefficient $\rho_{\cca}(X,Y)$, which is able to capture higher order dependencies. Interestingly, when $(X,Y)$ are jointly Gaussian, the two notions coincide. We next state this relation and provide a closed-form expression for the Gaussian mSMI.

\begin{proposition}[Gaussian mSMI]\label{prop:k_msmi_cca}
Let $X\sim \cN (m_X, \Sigma_X)$ and $Y\sim \cN(m_Y,\Sigma_Y)$ be $d_x$-- and $d_y$--dimensional jointly Gaussian vectors with nonsingular covariance matrices and cross-covariance $\Sigma_{XY}$. For any $k\leq d_x\wedge d_y$, we have
\vspace{-2mm}
\begin{equation}
\msmi_k(X;Y)=\sI(\rA^\intercal_{\cca}X;\rB^\intercal_{\cca} Y)=-\frac{1}{2}\sum_{i=1}^k\log\big(1-\sigma_i(\rT_{XY})^2\big),\label{eq:Gaussian_kSMI}
\end{equation}
\vspace{-2.5mm}

\noindent where $(\rA_{\cca},\rB_{\cca})$ are the CCA solutions from \eqref{eq:k_dim_cca}, $\rT_{XY}=\Sigma_X^{-1/2}\Sigma_{XY}\Sigma_Y^{-1/2}\in\RR^{d_x\times d_y}$, and $\sigma_k(\rT_{XY})\leq\ldots\leq \sigma_1(\rT_{XY})\leq 1$ are the top $k$ singular values of $\rT_{XY}$ (ordered).
\end{proposition}

This proposition is proven in Supplement \ref{subsec:k_msmi_cca_proof}. We first show that the optimization domain of $\msmi_k(X;Y)$ can be switched from the product of Stiefel manifolds to the space of all matrices subject to a unit variance constraint (akin to \eqref{eq:k_dim_cca}), without changing the mSMI value. This implies that the CCA solutions $(\rA_\cca,\rB_\cca)$ from \eqref{eq:k_dim_cca} are feasible for mSMI and we establish their optimality using a generalization of the Poincar\'e separation theorem \cite[Theorem 2.2]{rao1979separation}. Specializing \cref{prop:k_msmi_cca} to one-dimensional projections, i.e., when~$k=1$, the mSMI is given in terms of the canonical correlation coefficient $\rho_{\cca}(X,Y)\coloneqq \sup_{(\phi,\theta)\in\RR^{d_x}\times\RR^{d_y}}\rho(\thetat X,\phit Y)$. Namely, 
\[\msmi_1(X;Y)=\sI(\theta_\cca^\intercal X;\phi_\cca^\intercal Y)=-0.5\log\big(1-\rho_{\cca}(X,Y)^2\big),\]
where $(\theta_\cca,\phi_\cca)$ are the global optimizers of $\rho_{\cca}(X,Y)$.


\begin{remark}[Beyond Gaussian data]
    While the mSMI solution coincides with that of CCA in the Gaussian case, this is no longer expected to hold for non-Gaussian distributions. CCA is designed to maximize correlation, while mSMI has Shannon's mutual information between the projected variables as the optimization objective. Unlike correlation, mutual information captures higher order dependencies between the variables, and hence the optimal mSMI matrices will not generally coincide with $(\rA_\cca,\rB_\cca)$. Furthermore, the intricate dependencies captured by mutual information suggest that the optimal mSMI projections may yield representations that are more effective for downstream tasks. We empirically verify this observation in \cref{sec:experiments} on several tasks, including classifications, multi-view representation learning, and algorithmic fairness.    
\end{remark}

\begin{remark}[Max-sliced entropy and PCA]\label{rem:Gaussian_msh}
Similarly to the above, the Gaussian max-sliced entropy is related to PCA \cite{pearson1901liii}. In Supplement \ref{appendix:gaussian_pca}, we show that for ($d$-dimensional) $X\sim \cN(m,\Sigma)$, we have $\msh_k(X)= \sup_{\rA\in\sti(k,d)}\sh(A^\intercal X)= \sh(\rA_{\mathsf{PCA}}^\intercal X)=0.5\sum_{i=1}^k\log\big(2\pi e \lambda_i(\Sigma)\big)$, where $\rA_{\mathsf{PCA}}$ is optimal PCA matrix and $\lambda_1(\Sigma),\ldots\lambda_k(\Sigma)$ are the top $k$ eigenvalues of $\Sigma$ (which are non-negative since $\Sigma$ is a covariance matrix)\cite{pearson1901liii,hotelling1933analysis}. Extrapolating beyond the Gaussian case, this poses max-sliced entropy as an information-theoretic generalization of PCA for unsupervised dimensionality reduction. An analogous extension using the R\'enyi entropy of order $2$ was previously considered in~\cite{czarnecki2015maximum} for the purpose of downstream classification with binary labels. In that regard, $\msh_k(X)$ can be viewed as the $\alpha$-R\'enyi variant when $\alpha\to 1$.
\end{remark}

\subsection{Generalizations Beyond Linear Slicing}\label{subsec:nonlin_msi}

The notion of mSMI readily generalizes beyond linear slicing. Fix $d_x,d_y\geq 1$, $k\leq d_x\wedge d_y$, and consider two (nonempty) function classes $\cG\subseteq\{g:\RR^{d_x}\to\RR^k\}$ and $\cH\subseteq\{h:\RR^{d_y}\to\RR^k\}$.
\begin{definition}[Generalized mSMI]
The generalized mSMI between $(X,Y)\sim\mu_{XY}\in\cP(\RR^{d_x}\times\RR^{d_y})$ w.r.t. the classes $\cG$ and $\cH$ is $\msmi_{\cG,\cH}(X;Y)\coloneqq\sup_{(g,h)\in\cG\times\cH}\sI\big(g(X);h(Y)\big)$.
\end{definition}
The generalized variant reduces back to $\msmi_k(X;Y)$ by taking $\cG=\cG_\mathsf{proj}\coloneqq\{\mathfrak{p}^\rA:\,\rA\in\sti(k,d_x)\}$ and $\cH=\cH_\mathsf{proj}\coloneqq\{\mathfrak{p}^\rB:\,\rB\in\sti(k,d_y)\}$, 
but otherwise allows more flexibility in the way $(X,Y)$ are mapped into $\RR^k$. We also have that if $\cG\subseteq\cG'$ and $\cH\subseteq\cH'$, then $\msmi_{\cG,\cH}(X;Y)\leq \msmi_{\cG',\cH'}(X;Y)\leq \sI(X;Y)$, which corresponds to Property 1 from \cref{prop:msi_properties}. Further observations are as~follows. 

\begin{proposition}[Properties]\label{prop:gmsmi_properties}
    For any classes $\cG,\cH$, we have that $\msmi_{\cG,\cH}$ always satisfies Properties 3-5 from \cref{prop:msi_properties}. If further $\cG_\mathsf{proj}\subseteq\cG$ and $\cH_\mathsf{proj}\subseteq\cH$, then $\msmi_{\cG,\cH}$ also satisfies Property 2.
\end{proposition}
We omit the proof as it follows by the same argument as \cref{prop:msi_properties}, up to replacing the linear projections with the functions $(g,h)\in\cG\times \cH$. In practice, the classes $\cG$ and $\cH$ are chosen to be parametric, typically realized by artificial NNs. As discussed in \cref{rem:nonlin_slicing_NE} ahead, this is well-suited to the neural estimation framework for mSMI (both standard and generalized). 
Lastly, note that $\msmi_{\cG,\cH}(X;Y)$ corresponds to the objective of multi-view representation learning \cite{tschannen2019mutual}, which considers the maximization of the mutual information between NN-based representation of the considered variables. We further investigate this relation in Section \ref{sec:exp_representation}.

\section{Neural Estimation of Max-SMI}\label{sec:neural_est}

We study estimation of mSMI from data, seeking an efficiently computable and scalable approach subject to formal performance guarantees. Towards that end, we observe that the mSMI is compatible with neural estimation \cite{belghazi2018mutual,sreekumar2022neural} since it has a natural variational form. In what follows we derive the neural estimator, describe the algorithm to compute it, and provide non-asymptotic error bounds.

\subsection{Estimator and Algorithm}\label{subsec:NE_est_alg}

Fix $d_x,d_y\geq 1$, $k\leq d_x\wedge d_y$, and $\mu_{XY}\in\cP(\RR^{d_x}\times \RR^{d_y})$; we suppress $k,d_x,d_y$ from our notation of the considered function classes. Neural estimation is based on the DV variational form:\footnote{One may instead use the form that stems from convex duality:
$\sI(U;\mspace{-1mu}V)\mspace{-2mu}=\mspace{-1mu}\sup_f\EE[f(U,V)]-\EE\big[e^{f(\tilde U,\tilde V)}-1\big]$.}
\[
\sI(X;Y)=\sup_{f\in\cF}\ldv(f;\,\mu_{XY}),\quad \ldv(f;\,\mu_{XY})\coloneqq\EE[f(X,Y)]-\log\big(e^{\EE[f(\tilde{X},\tilde{Y})]}\big),
\]
where $(X,Y)\sim \mu_{XY}$, $(\tilde{X},\tilde{Y})\sim \mu_X\otimes\mu_Y$, and $\cF$ is the class of all measurable functions $f:\RR^{d_x}\times\RR^{d_y}\to\RR$ (often referred to as DV potentials) for which the expectations above are finite. As mSMI is the maximal mutual information between projections of $X,Y$, we have 
\[
\msmi_k(X;Y)=\sup_{(\rA ,\rB)\in\sti(k,d_x)\times \sti(k,d_y)} \sup_{f\in\cF}\ldv\big(f;\,(\mathfrak{p}^\rA,\mathfrak{p}^\rB)_\sharp\mu_{XY}\big)=\sup_{f\in\cF^{\mathsf{\mspace{1mu}proj}}}\ldv(f;\mu_{XY}),
\]
where $\cF^{\mathsf{\mspace{1mu}proj}}\coloneqq\big\{f\circ(\mathfrak{p}^\rA,\mathfrak{p}^\rB)\,:\,f\in\cF,\,(\rA,\rB)\in\sti(k,d_x)\times\sti(k,d_y)\big\}$. The RHS above is given by optimizing the DV objective $\ldv$ over the \emph{composed} class $\cF^{\mathsf{\mspace{1mu}proj}}$, which first projects $(X,Y)\mapsto (\rA^\intercal X,\rB^\intercal Y)$ and then applies a DV potential $f:\RR^k\times\RR^k\to\RR$ to the projected variables.

\paragraph{Neural estimator.} Neural estimators parametrize the DV potential by NNs, approximate expectations by sample means, and optimize the resulting empirical objective over parameter space. Let $\cF_{\mathsf{nn}}$ be a class of feedforward NNs with input space $\RR^k\times \RR^k$ and real-valued outputs.\footnote{For now, we leave the architecture (number of layers/neurons, parameter bounds, nonlinearity) implicit to allow flexibility of implementation; we will specialize to a concrete class when providing theoretical guarantees.} Given i.i.d. samples $(X_1,Y_1),\ldots,(X_n,Y_n)$ from $\mu_{XY}$, we first generate negative samples (i.e., from $\mu_X\otimes\mu_Y$) by taking $(X_1,Y_{\sigma(1)}),\ldots,(X_n,Y_{\sigma(n)})$, where $\sigma\in S_n$ is a permutation such that $\sigma(i)\neq i$, for all $i=1,\ldots,n$. The neural estimator of $\msmi_k(X;Y)$ is now given by 
\begin{equation}
\widehat{\mathsf{SI}}_k^{\cF_{\mathsf{nn}}}(X^n,Y^n):=\sup_{f \in \cF_{\mathsf{nn}}^{\mathsf{proj}}} \frac 1n \sum_{i=1}^n f(X_i,Y_i)- \log\left(\frac 1n \sum_{i=1}^n e^{f(X_i,Y_{\sigma(i)})}\right),\label{eq:mSMI_NE}
\end{equation}
where $\cF_{\mathsf{nn}}^{\mathsf{proj}}\coloneqq\big\{ f\circ(\mathfrak{p}^\rA,\mathfrak{p}^\rB)\,:\,f\in\cF_{\mathsf{nn}},\,(\rA,\rB)\in\sti(k,d_x)\times\sti(k,d_y)\big\}$ is the composition of the NN class with the projection maps. The projection maps can be absorbed into the NN architecture as a first linear layer that maps the $(d_x+ d_y)$-dimensional input to a $2k$-dimensional feature vector, which is then further processed by the original NN $f\in\cF_{\mathsf{nn}}$. Note that projection onto the Stiefel manifold can be efficiently and differentiably computed via QR decomposition. Hence, the Stiefel manifold constraint can be easily enforced by setting $A, B$ to be the projections of unconstrained $d \times k$ matrices. Further details on the implementation are given in Supplement \ref{appendix:imeplementation}.

\begin{remark}[Nonlinear slicing]\label{rem:nonlin_slicing_NE}
For learning tasks that may need more expressive representations of $(X,Y)$, one may employ the nonlinear mSMI variant from \cref{subsec:nonlin_msi}. In practice, the classes $\cG=\{g_\theta\}$ and $\cH=\{h_\phi\}$ are taken to be parametric, realized by NNs. These NNs naturally compose with the DV critic $f_\psi$, resulting in a single compound model $f_\psi\circ(g_\theta,h_\phi)$ that is optimized together. 
\end{remark}

\subsection{Performance Guarantees}\label{subsec:NE_guarantees}

Neural estimation nominally involves three sources of error: (i) function approximation of the DV potential; (ii) empirical estimation of the means; and (iii) optimization, which comes from employing suboptimal (gradient-based) routines. Our analysis provides sharp non-asymptotic bounds for errors of type (i) and (ii), leaving the account of the optimization error for future work. We focus on a class of $\ell$-neuron shallow ReLU networks, although the ideas extend to other nonlinearities and deep architectures. Define $\cF_{\mathsf{nn}}^{\ell}$ as the class of NNs $f:\RR^k\times \RR^k \rightarrow \mathbb{R},\, f(z)=\sum\nolimits_{i=1}^\ell \beta_i \phi\left(\langle w_i, z\rangle+b_i\right)+\langle w_0, z\rangle + b_0$, whose parameters satisfy $\max_{1 \leq i \leq \ell}\|w_{i}\|_1 \vee |b_i| \leq 1$, $\max_{1 \leq i \leq \ell}|\beta_i| \leq \frac{a_\ell}{2\ell}$, and $|b_0|,\|w_0\|_1 \leq a_\ell$, 
where $\phi(z)=z\vee0$ is the ReLU activation and $a_\ell=\log \log \ell \vee 1$. 

Consider the neural mSMI estimator $\widehat{\mathsf{SI}}_k^{n,\ell}\coloneqq\widehat{\mathsf{SI}}_k^{\cF^\ell_{\mathsf{nn}}}(X^n,Y^n)$ (see \eqref{eq:mSMI_NE}). We provide convergence rates for it over an appropriate distribution~class, drawing upon the results of \cite{sreekumar2021non} for neural estimation of $f$-divergences. For compact $ \cX \subset \RR^{d_x}$ and $\cY\subset \RR^{d_y}$, let $\cP_{\s{ac}}(\cX\times\cY)$ be the set of all Lebesgue absolutely continuous joint distribution $\mu_{XY}$ with $\supp(\mu_{XY})\subseteq\cX\times\cY$. Denote the Lebesgue density of $\mu_{XY}$ by $f_{XY}$. The distribution class of interest~is\footnote{Here, $\cC^s_b(\cU)\coloneqq\{f \in \cC^s(\cU): \max_{\alpha: \|\alpha\|_1 \leq s} \|D^\alpha  f \|_{\infty,\cU} \leq  b  \}$, where $D^\alpha$, $\alpha=(\alpha_1, \ldots, \alpha_d) \in \ZZ^d_{\geq 0}$, is the partial derivative operator of order $\sum_{i=1}^d\alpha_i$.  The restriction of $f: \RR^d \to \RR$ to $\cX \subseteq \RR^d$ is $f|_{\cX}$.
}
\begin{equation}\label{eq:PKL}
\cP_k(M,b)\coloneqq\left\{ \mu_{XY}\in \cP_{\s{ac}}(\cX \times\cY):\begin{aligned}
    &  \exists\, r \in\cC^{k+3}_{b}(\cU) \ \textnormal{for some open set }\cU \supset \cX \times \cY
    \\ & \  \textnormal{s.t.} \ \log f_{XY} = r|_{\cX \times \cY}, \   \sI(X;Y) \leq M \end{aligned}\right\},
\end{equation}
which, in particular, contains distributions whose densities are bounded from above and below on $\cX\times\cY$ with a smooth extension to an open set covering $\cX\times\cY$. This includes uniform distributions, truncated Gaussians, truncated Cauchy distributions, etc. 
The following theorem provides the convergence rate for the mSMI neural estimator, uniformly over $\cP_k(M,b)$. 
\begin{theorem}[Neural estimation error]\label{thm:msmi_NE}
For any $M,b\geq 0$, we have
\[
  \sup_{\mu_{X,Y}\in\cP_k(M,b)}\mathbb{E}\left[\left|\msmi_k(X;Y)- \widehat{\mathsf{SI}}_k^{n,\ell}\right|\right]  \leq C k^{\frac 12}\big(\ell^{-\frac{1}{2}}+ kn^{-\frac 12}\big).
\]
where the constant $C$ depends on $M$, $b$, $k$, and the radius of the ambient space, which is given by $\|\cX\times\cY\|\coloneqq \sup_{(x,y)\in \cX\times\cY}\|(x,y)\|$.
\end{theorem} 

The theorem is proven in Supplement \ref{appendix:msmi_NE} by adapting the error bound from \cite[Proposition~2]{sreekumar2022neural} to hold for $\sI(\rA^\tr X;\rB^\tr Y)$, uniformly over $(\rA,\rB)\in\sti(k,d_x)\times\sti(k,d_y)$. To that end, we show that for any $\mu_{XY}\in\cP_k(b,M)$, the log-density of $(\rA^\intercal X,\rB^\intercal Y)\sim (\mathfrak{p}^\rA,\mathfrak{p}^\rB)_\sharp\mu_{XY}$ admits an extension (to an open set containing the support) with $k+3$ continuous and uniformly bounded derivatives.

\begin{remark}[Parametric rate and optimality]
Taking $\ell\asymp n$, the resulting rate in \cref{thm:msmi_NE} is parametric, and hence minimax optimal. This result implicitly assumes that $M$ is known when picking the neural net parameters. This assumption can be relaxed to mere existence of (an unknown) $M$, resulting in an extra $\mathrm{polylog}(\ell)$ factor multiplying the $n^{-1/2}$ term.
\end{remark}

\begin{remark}[Comparison to average-SMI]
Neural estimation of classic mutual information under the framework of \cite{sreekumar2022neural} requires the density to have H\"older smoothness $s\geq \lfloor(d_x +d_y)/2 \rfloor+3$. For $\msmi_k(X;Y)$, smoothness of $k+3$ is sufficient (even though the ambient dimension is the same), which means it can be estimated over a larger class of distributions. Similar gains in terms of smoothness levels was observed for aSMI in \cite{goldfeld2022k}. Nevertheless, we note that mSMI is significantly more compatible with neural estimation than average-slicing \cite{goldfeld2021sliced,goldfeld2022k}. The mSMI neural estimator integrates the max-slicing into the NN architecture and optimizes a single objective. The aSMI neural estimator from \cite{goldfeld2022k} requires an additional Monte Carlo integration step to approximate the integral over the Steifel manifolds. This results in an extra $k^{1/2}m^{-1/2}$ term in the error bound, where $m$ is the number of Monte Carlo samples, introducing significant computational overhead (see \cref{subsec:experiments_neural_est}).
\end{remark}

\begin{remark}[Non-ReLU networks]
    Theorem \ref{thm:msmi_NE} employs the neural estimation bound from \cite{sreekumar2022neural}, which, in particular, relies on \cite{klusowski2018approximation} to control the approximation error. As noted in \cite{sreekumar2022neural}, their bound extends to any other sigmoidal bounded activation with $\lim_{z\to-\infty}\sigma(z)=0$ and $\lim_{z\to\infty}\sigma(z)=1$ by appealing to the approximation bound from \cite{barron1993universal} instead. 
    Doing so would allow relaxing the smoothness requirement on the extension to $r\in\cC^{k+2}_{b}$ in \eqref{eq:PKL}, albeit at the expense of scaling the hidden layer parameters as $\ell^{1/2}\log \ell$. This stands in contrast to the presented ReLU-based bound, where the parameters are bounded independently of $\ell$. 
\end{remark}


\section{Experiments}\label{sec:experiments}

\subsection{Neural Estimation}\label{subsec:experiments_neural_est}
\begin{wrapfigure}{R}{0.45\textwidth}
    \centering
    \vspace{-4mm}
    \begin{subfigure}{0.45\textwidth}
        \centering
        \begin{subfigure}{\textwidth}
            \centering
            \scalebox{0.65}{\input{Figures/NE_k1/NE_k1_asi_rho05}}
            \label{fig:layer_ablation1}
        \end{subfigure}
        
        \begin{subfigure}{\textwidth}
            \centering
            \scalebox{0.65}{\input{Figures/NE_k1/NE_k1_maxgrad_rho0.5}}
            \label{fig:asi_msi_ne}
        \end{subfigure}
    \end{subfigure}%
    \hfill
    \begin{subfigure}{0.45\textwidth}
        \centering
        \scalebox{0.65}{\input{Figures/NE_k1/NE_k1_time_log_rho0.5}}
        \label{fig:epoch_time_ne}
    \end{subfigure}
    \caption{Neural estimation performance with $\rho=0.5$. Convergence vs. $n$ in upper figures and average epoch time vs. $n$ in lower figure.}
    \vspace{-0.3in}
    \label{fig:NE}
\end{wrapfigure}
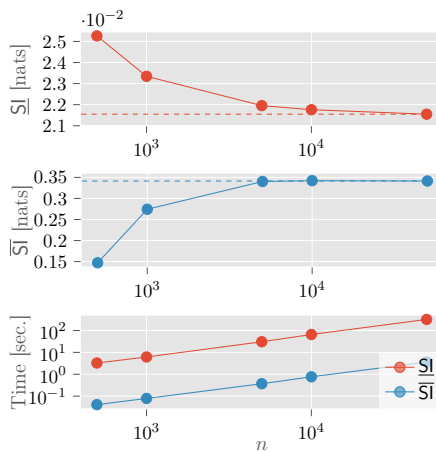


We compare the performance of neural estimation methods for mSMI and aSMI on a synthetic dataset of correlated Gaussians. Namely, let $X,Z\sim \cN(0,1)$ be i.i.d. and set $Y=\rho X+\sqrt{1-\rho^2} Z$, for $\rho\in(0,1)$. The goal is to estimate mSMI and aSMI with $k$-dimensional projections between $(X,Y)$. 
We train our mSMI neural estimator and the aSMI neural estimator from \cite[Section~4.2]{goldfeld2022k} based on $n$ i.i.d. samples, and compare their performance as a function of $n$. Both average and max-sliced algorithms converge at similar rates; however, aSMI has significantly higher time complexity due to the need to train multiple neural estimators (one for each projection direction). This is shown in Figure \ref{fig:NE}, where we compare the average epoch time for each algorithm against the dataset size. Implementation details are given in Supplement \ref{appendix:imeplementation}.

\subsection{Independence Testing}

In this experiment, we compare mSMI and aSMI for independence testing. We follow the setting from \cite[Section 5]{goldfeld2022k}, generating $d$-dimensional samples correlated in a latent $d'$-dimensional subspace and estimating the information measure to determine dependence. We estimate the aSMI with the method from \cite{goldfeld2022k}, using $m=1000$ Monte Carlo samples and the Kozachenko-Leonenko estimator for the mutual information between the projected variables \cite{kraskov2004estimating}. We then compute AUC-ROC over 100 trials, considering various ambient and projected dimensions. For mSMI, as we cannot differentiate through the Kozachenko-Leonenko estimator, we resort to gradient-free methods. Specifically, we employ the LIPO algorithm from \cite{malherbe2017global} with a stopping criterion of 1000 samples. This choice is motivated by the Lipschitzness of $(\rA,\rB)\mapsto \sI(\rAt X;\rBt Y)$ w.r.t. the Frobenius norm on $\sti(k,d_x)\times\sti(k,d_y)$ (cf. \cite[Lemma~2]{goldfeld2022k}).
Figure \ref{fig:independence} shows that when $k > 2$, mSMI better captures independence than aSMI, particularly in the lower sample regime. We hypothesize that this may be due to the fact that the shared signal lies in a low-dimensional subspace, which mSMI can isolate and perhaps better exploit than aSMI, which averages over all subspaces. When $k$ is much smaller than the shared signal dimension $d'$, mSMI will fail to capture all the information and may underperform aSMI which takes all slices into account. 
Results are averaged over 10 seeds.
Further implementation details are in Supplement~\ref{appendix:imeplementation}.

\begin{figure}[t]
    \centering
    \hspace{-1cm}
    \begin{subfigure}[b]{0.13\textwidth}
        \centering
        \scalebox{0.9}{\input{Figures/independence_comparison_d10_dlatent_3/indep_k1}}
    \end{subfigure}%
    \hspace{0.8cm}
    \begin{subfigure}[b]{0.13\textwidth}
        \centering
        \scalebox{0.9}{\input{Figures/independence_comparison_d10_dlatent_3/indep_k2}}
    \end{subfigure}
    \hspace{0.3cm}
    \begin{subfigure}[b]{0.13\textwidth}
        \centering
        \scalebox{0.9}{\input{Figures/independence_comparison_d10_dlatent_3/indep_k3}}
    \end{subfigure}%
    \hspace{0.3cm}
    \begin{subfigure}[b]{0.13\textwidth}
        \centering
        \scalebox{0.9}{\input{Figures/independence_comparison_d20_dlatent6/indep_k3}}
    \end{subfigure}%
    \hspace{0.3cm}
    \begin{subfigure}[b]{0.13\textwidth}
        \centering
        \scalebox{0.9}{\input{Figures/independence_comparison_d20_dlatent6/indep_k4}}
    \end{subfigure}%
    \hspace{0.3cm}
    \begin{subfigure}[b]{0.13\textwidth}
        \centering
        \scalebox{0.9}{\input{Figures/independence_comparison_d20_dlatent6/indep_k5}}
    \end{subfigure}%
    \caption{ROC-AUC comparison. Dashed and solid lines show results for aSMI \cite{goldfeld2022k} and mSMI (ours), respectively. Blue plots correspond to $(d,d')=(10,4)$, while red correspond to $(d,d')=(20,6)$.}
    \label{fig:independence}
    \vspace{-3mm}
\end{figure}
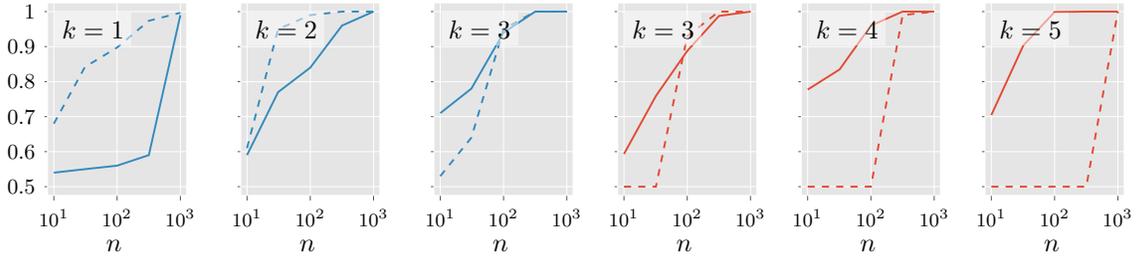

\subsection{Multi-View Representation Learning}\label{sec:exp_representation}

\begin{wraptable}{r}{0.5\textwidth}
\vspace{-4mm}
\centering
\resizebox{\linewidth}{!}{
\begin{tabular}{|c|c|c|c|c|}
\toprule
\textbf{$k$} & \textbf{Linear CCA} & \textbf{Linear mSMI} & \textbf{MLP DCCA} & \textbf{MLP mSMI} \\
\midrule
1 & 0.261$\pm$0.03 & \textbf{0.274}$\pm$0.02 & 0.284$\pm$0.03 & \textbf{0.291}$\pm$0.02 \\
2 & 0.32$\pm$0.02 & \textbf{0.346}$\pm$0.02 & 0.314$\pm$0.03 & \textbf{0.417}$\pm$0.02 \\
4 & 0.42$\pm$0.01 & \textbf{0.478}$\pm$0.02 & 0.441$\pm$0.04 & \textbf{0.546}$\pm$0.01 \\
8 & 0.553$\pm$0.03 & \textbf{0.666}$\pm$0.01 & 0.645$\pm$0.02 & \textbf{0.665}$\pm$0.01 \\
12 & 0.614$\pm$0.02 & \textbf{0.751}$\pm$0.01 & 0.697$\pm$0.01 & \textbf{0.753}$\pm$0.01 \\
16 & 0.673$\pm$0.02 & \textbf{0.775}$\pm$0.01 & 0.730$\pm$0.02 & \textbf{0.779}$\pm$0.01 \\
20 & 0.704$\pm$0.007 & \textbf{0.79}$\pm$0.006 & 0.774$\pm$0.01 & \textbf{0.798}$\pm$0.01 \\
\bottomrule
\end{tabular}}
\caption{Downstream classification test accuracy based on MNIST representations learned by CCA and mSMI.}
\label{table:accuracy}
\end{wraptable}

We next explore mSMI as an information-theoretic generalization of CCA by examining its utility in multi-view representation learning---a popular CCA application. Without using class labels, we obtain mSMI-based $k$-dimensional representations of the top and bottom halves of MNIST images (considered as two separate views of the digit image). This is done by computing the $k$-dimensional mSMI between the views and using the maximizing projected variables as the representations. We compare to similarly obtained CCA-based representations, following the method of \cite{andrew2013deep}. Both linear and nonlinear (parameterized by an MLP neural network) slicing models are optimized with similar initialization and data pipelines but different loss functions. Performance is evaluated via downstream 10-class classification accuracy utilizing the learned top-half representations. Results are averaged over 10 seeds. As shown in Table \ref{table:accuracy}, mSMI outperforms CCA for learning meaningful representations. Interestingly, linear representations learned by mSMI outperform nonlinear representations from the CCA methodology, demonstrating the potency of mSMI. Full implementation details and additional results are given in Supplements \ref{appendix:imeplementation} and \ref{appendix:reprensetaion}, 
respectively.

We note that aSMI is not considered for this experiment since it does not provide a concrete latent space representation (due to it being an averaged quantity). Moreover, if one were to attempt to maximize aSMI as an objective to derive such a concrete representation, this would simply lead back to  computing mSMI;  cf. Remark \ref{rem:msmi_vs_asmi}.

\subsection{Learning Fair Representations}

\begin{wraptable}{R}{0.7\textwidth}
\vspace{-.15in}
\caption{Learning a fair representation of the US Census Demographic dataset, following the setup of \cite{chen2023scalable}. Results are shown as the median over 10 runs with random data splits. The fairest result is $k=6$. }
\label{table:Census}
\centering
\resizebox{\linewidth}{!}{
\begin{tabular}{|c|c|c|c|c|c|c|c|c|c|}
\hline
& \textbf{N/A} & \textbf{Slice \cite{chen2023scalable}} & \multicolumn{7}{c|}{\textbf{mSMI (ours)}}  \\
\hline
\cellcolor{gray}& \cellcolor{gray}& \cellcolor{gray} & $k=1$ & $k=2$ & $k=3$ & $k=4$ & $k=5$ & $\mathbf{k=6}$ & $k=7$\\
\hline
$\rho_\mathsf{HGR}(Z,Y) \uparrow$ & 0.949 & 0.967 & 0.955 & 0.958 & 0.952 & 0.942 & 0.940 & 0.957 & 0.933 \\
\hline
$\rho_\mathsf{HGR}(Z,A) \downarrow$ & 0.795 & 0.116 & 0.220 & 0.099 & 0.067 & 0.048 & 0.029 & \textbf{0.026} & 0.047 \\
\hline
\end{tabular}}
\end{wraptable}

Another common application of dependence measures is learning fair representations of data. We seek a data transformation $Z = f(X)$ that is useful for predicting some outcome or label $Y$, while being statistically independent of some sensitive attribute $A$ (e.g., gender, race, or religion of the subject). In other words, a fair representation is one that is not affected by the subjects' protected attributes so that downstream predictions are not biased against protected groups, even if the training data may have been biased. Following the setup of \cite{chen2023scalable}, we measure utility and fairness using the HGR maximal correlation $\rho_\mathsf{HGR}(\cdot,\cdot) = \sup_{h,g} \rho\big(h(\cdot),g(\cdot)\big)$, seeking large $\rho_\mathsf{HGR}(Z,Y)$ and small $\rho_\mathsf{HGR}(Z,A)$ where $h$ and $g$ are parameterized by NNs. As solving this minimax problem directly is difficult in practice, following \cite{chen2023scalable} we learn $Z$ by optimizing the bottleneck equation
$\rho_\mathsf{HGR}(Z,Y) - \beta \msmi_k (Z,A)$,
where we use a neural estimator for the mSMI and $\beta$, $k$ are hyperparameters. 

Table \ref{table:Census} shows results on the US Census Demographic dataset extracted from the 2015 American Community Survey, which has 37 features collected over 74,000 census tracts. Here $Y$ is the fraction of children below the poverty line in a tract, and $A$ is the fraction of women in the tract. Following the same experimental setup as \cite{chen2023scalable}, the learned $Z$ is 80-dimensional. As \cite{chen2023scalable} showed that their ``Slice'' approach significantly outperformed all other baselines on this experiments under a computational constraint\footnote{Runtime per iteration not to exceed the runtime of Slice per iteration. We used an NVIDIA V100 GPU.}, we apply the same computational constraint to our approach and compare only to Slice and to the ``N/A'' fairness-agnostic model trained on the bottleneck objective with $\beta = 0$. 
Note that for $k > 1$, mSMI learns a more fair representation $Z$ (lower $\rho_\mathsf{HGR}(Z,A)$) than Slice, while retaining a utility $\rho_\mathsf{HGR}(Z,Y)$ on par with the fairness agnostic N/A model.
We emphasize that due to the reasons outlined in Section \ref{sec:exp_representation}, aSMI is not suitable for the considered task and is thus not included in the comparison.
Results on the Adult dataset are shown in Supplement \ref{app:adult}.

\subsection{Max-Sliced InfoGAN}
We present an application of max-slicing to generative modeling under the InfoGAN framework \cite{chen2016infogan}. The InfoGAN learns a disentangled latent space by maximizing the mutual information between a latent code variable and the generated data. We revisit this architecture but replace the classical mutual information regularizer in the InfoGAN objective with mSMI. Our max-sliced InfoGAN is tested on the MNIST and Fashion-MNIST datasets. Figure \ref{fig:infogan} presents the max-sliced InfoGAN performance for several projection dimensions. We consider $3$ latent codes $(C_1,C_2,C_3)$, which automatically learn to encode different features of the generated data. We vary the values of $C_1$, which is a $10$-state discrete variable, along the column (and consider random values of $(C_2,C_3)$ along the rows). Evidently, $C_1$ has successfully disentangled the 10 class labels and the quality of generated samples is on par with past implementations \cite{chen2016infogan,goldfeld2022k}. It is noteworthy that since mSMI relies on low-dimensional projections, the resulting InfoGAN uses a reduced number of parameters (at the negligible expense of optimizing over the linear projections). Additional implementation details are found in Supplement \ref{appendix:imeplementation}.


\begin{figure}[b]
    \centering
    \begin{subfigure}[b]{0.3\textwidth}
        \centering
        \includegraphics[width=\textwidth]{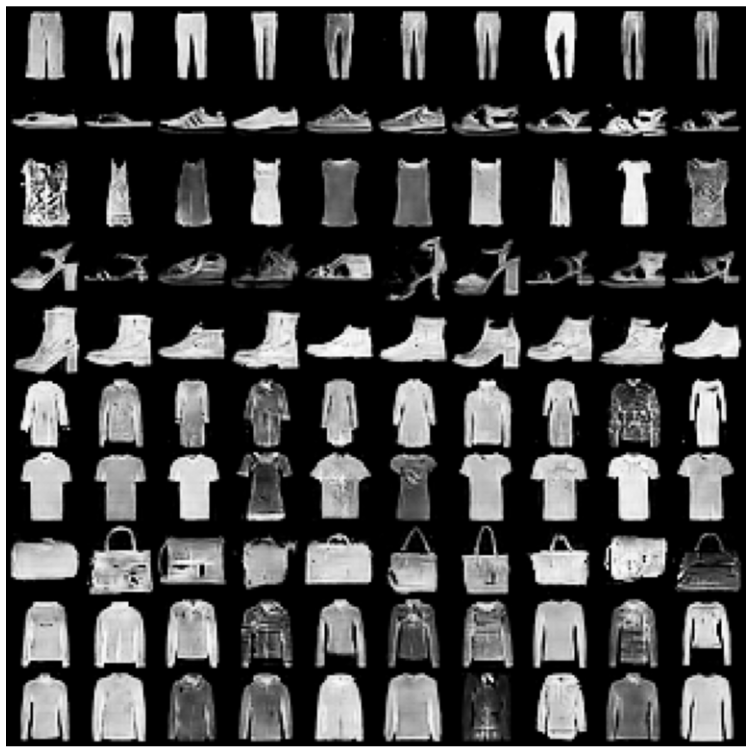}
        \caption{FashionMNIST}
    \end{subfigure}%
    \hspace{0.23cm}
    \begin{subfigure}[b]{0.301\textwidth}
        \centering
        \includegraphics[width=\textwidth]{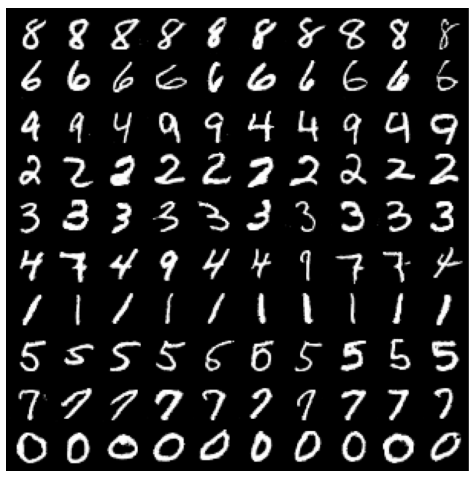}
        \caption{MNIST, $k=5$}
    \end{subfigure}
    \hspace{0.1cm}
    \begin{subfigure}[b]{0.3063\textwidth}
        \centering
        \includegraphics[width=\textwidth]{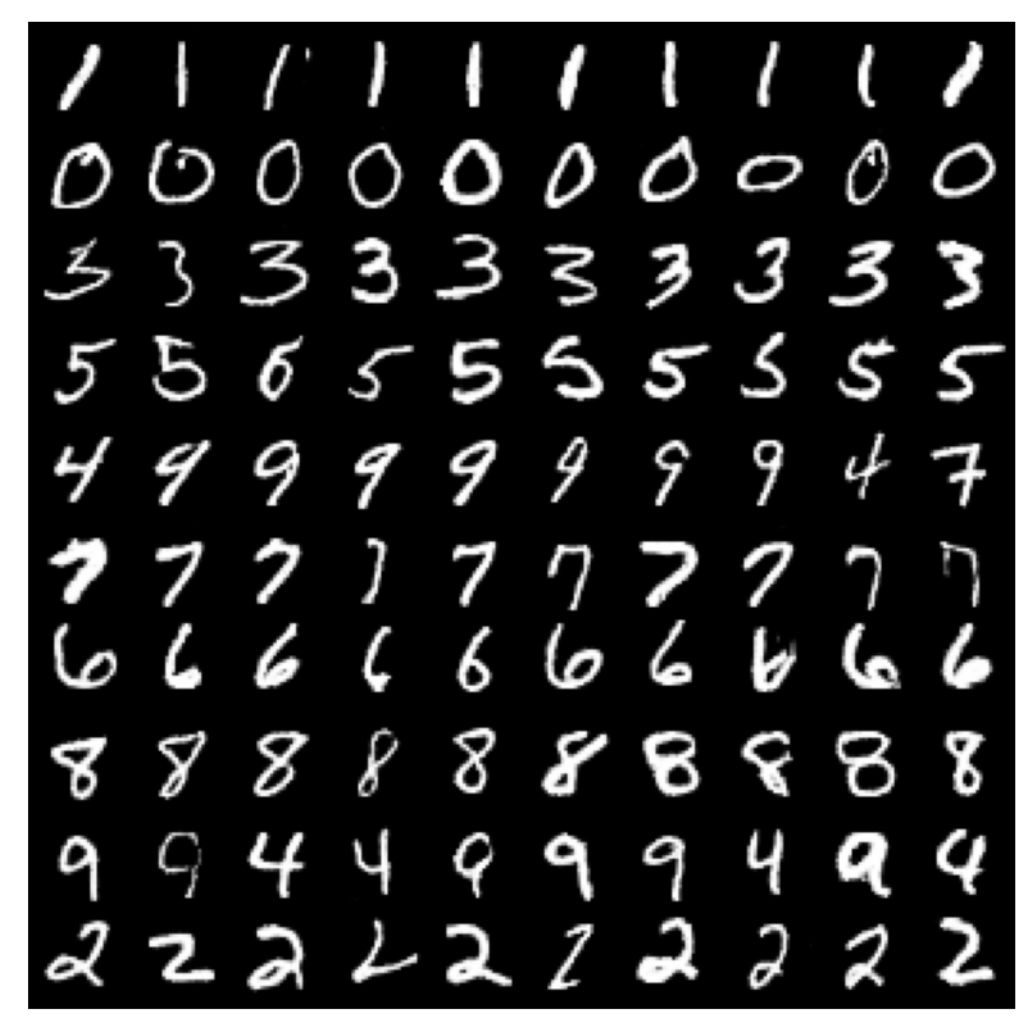}
        \caption{MNIST, $k=10$}
    \end{subfigure}%
    \caption{MNIST images generated via the max-sliced InfoGAN.}
    \label{fig:infogan}
\end{figure}

\vspace{-1mm}
\section{Conclusion}
\vspace{-1mm}

This paper proposed mSMI, an information theoretic generalization of CCA.
mSMI captures the full dependence structure between two high dimensional random variables, while only requiring an optimized linear projection of the data.
We showed that mSMI inherits important properties of Shannon's mutual information and that when the random variables are Gaussian, the mSMI optimal solutions coincide with classic $k$-dimensional CCA.
Moving beyond Gaussian distributions, we present a neural estimator of mSMI and establish non-asymptotic error bounds.

Through several experiments we demonstrate the utility of mSMI for tasks such as independence testing, multi-view representation learning, algorithmic fairness and generative modeling, showing it outperforms popular methodologies.
Possible future directions include an investigation of an operational meaning of mSMI, either in information theoretic or physical terms, extension of the proposed formal guarantees to the nonlinear setting, and the extension of the neural estimation convergence guarantees to deeper networks. Additionally, mSMI can provide a mathematical foundation to mutual information-based representation learning, a popular area of self-supervised learning \cite{balestriero2023cookbook,shwartz2023compress}.

In addition to the above, we plan to develop a rigorous theory for the choice of $k$, which is currently devised empirically and is treated as a hyperparameter.
When the support of the distributions lies in some $d'<d$ dimensional subspace, the choice of $k=d'$ is sufficient to recover the classical mutual information, and therefore it characterizes the full dependence structure. Extrapolating from this point, we conjecture that the optimal value of $k$ is related to the intrinsic dimension of the data distribution, even when it is not strictly supported on a low-dimensional subset.


\clearpage
\bibliographystyle{unsrt}
\bibliography{main_neurips.bib}
\newpage

\newpage

\clearpage
\iftrue
\setcounter{page}{1}



\appendix
\section{Proofs}
\subsection{Proof of Proposition \ref{prop:msi_properties}}\label{proof:msi_properties}
We note that \textbf{1} is restated and was proved in \cite[Appendix~A.1]{goldfeld2021sliced}
\paragraph{Proof of 2:}
Non-negativity directly follows by non-negativity of mutual information.
Equivalence directly follows from the bound.
When $X\indep Y$ we have $I(X;Y)=0$ and therefore $\msmi_k(X;Y)\leq0$, which implies $\msmi_k(X;Y)=0$ due to non-negativity.
When $\msmi_k(X;Y)=0$ we have $\asmi_k(X;Y)=0$, which implies $X\indep Y$.
\paragraph{Proof of 3:} The representation immediately follows from the representation of mutual information $\sI(X;Y)=\dkl(\mu_{XY}\|\mu_X\otimes\mu_Y)$.
The expressions follows by plugging it into the definition of mSMI (\eqref{eq:msmi_intro} in the main text).
\paragraph{Proof of 4:} For the triplet $(X_1,X_2,Y)$, we have
\begin{align*}
    \msmi_k(X_1,X_2;Y) &= \max_{\rA,\rB,\psi}\sI(\rA^\intercal X_1, \rB^\intercal X_2 ;\psit Y)\\
    &= \max_{\rA,\rB,\psi} \left(\sI(\rA^\intercal X_1;\psit Y) + \sI(\rB^\intercal X_2 ; \psit Y | \rA^\intercal X_1)\right) \\
    &\leq \max_{\rA,\psi} \sI(\rA^\intercal X_1;\psit Y) + \max_{\rA,\rB,\psi}\sI(\rB^\intercal X_2 ; \psit Y | \rA^\intercal X_1) \\
    &= \msmi_k(X_1;Y)+\msmi_k(X_2;Y|X_1),
\end{align*}
where the penultimate inequality follows from the properties of the maximum function.
This proof straightforward generalizes to $n$ variables.
\paragraph{Proof of 5:} The proof relies on the independence of functions of independent random variables. We have 
\begin{align*}
    \msmi_k(X^n,Y^n) &= \max_{\rA_1,\dots,\rA_n,\rB_1,\dots,\rB_n}\sI(\rA^\intercal_1X_1,\dots,\rA^\intercal_n X_n;\rB^\intercal_1Y_1,\dots,\rB^\intercal_n Y_n)\\
    &= \max_{\rA_1,\dots,\rA_n,\rB_1,\dots,\rB_n}\left(\sum_{i=1}^n \sI(\rA^\intercal_i X_i;\rB^\intercal_1Y_1,\dots,\rB^\intercal_n Y_n|\rA^\intercal_1 X_1,\dots,\rA^\intercal_{i-1} X_{i-1}) \right)\\
    &= \max_{\rA_1,\dots,\rA_n,\rB_1,\dots,\rB_n}
    \left(\sum_{i=1}^n \sum_{j=1}^n \sI(\rA^\intercal_i X_i;\rB^\intercal_jY_j|\rA^\intercal_1 X_1,\dots,\rA^\intercal_{i-1} X_{i-1}, \rB^\intercal_1Y_1,\dots,\rB^\intercal_{j-1} Y_{j-1}) \right)\\
    &= \max_{\rA_1,\dots,\rA_n,\rB_1,\dots,\rB_n}\left(\sum_{i=1}^n \sum_{j=1}^n \sI(\rA^\intercal_i X_i;\rB^\intercal_jY_j)
    \right)\\
    &= \sum_{i=1}^n \max_{\rA_i,\rB_i}\sI(\rA^\intercal_i X_i;\rB^\intercal_iY_i),
\end{align*}
where the last inequality follows from the independence of the maximal $\sI(\rA_i^\intercal X_i,\rA_i^\intercal Y_i)$ in $(\rA_j,\rB _j)$.


This concludes the proof. $\hfill\square$

\subsection{Proof of \cref{prop:k_msmi_cca}}\label{subsec:k_msmi_cca_proof}

We first show that the $k$-dimensional Gaussian mSMI can be realized as an optimization of the projected mutual information over the same domain as the CCA problem from~\eqref{eq:k_dim_cca}. This equivalence follows by invariance of mutual information to bijections.  

\begin{lemma}\label{lemma:Gaussian_msmi_domain}
For jointly Gaussian $X\sim \cN (m_X, \Sigma_X)$ and $Y\sim \cN(m_Y,\Sigma_Y)$ with cross-covariance $\Sigma_{XY}$, we have
\begin{align*}
\msmi_k(X;Y)&=\sup_{\substack{(\rA,\rB)\in \RR^{d_x\times k}\times\RR^{d_y\times k}:\\ \rA^\intercal \Sigma_X\rA=\rB^\intercal\Sigma_Y\rB=\rI_k}}\sI(\rA^\intercal X;\rB^\intercal Y)\\
&=\sup_{\substack{(\rA,\rB)\in \RR^{d_x\times k}\times\RR^{d_y\times k}:\\ \rA^\intercal \Sigma_X\rA=\rB^\intercal\Sigma_Y\rB=\rI_k}}\mspace{-12mu}-\frac{1}{2}\log\det\big(\rI_k-(\rA^\intercal\Sigma_{XY}\rB)^\intercal(\rA^\intercal\Sigma_{XY}\rB)\big).
\end{align*}
\end{lemma}

\begin{proof}
By translation invariance of mutual information, we may assume w.l.o.g. that the means are zero, i.e., $m_X=m_Y=0$. The mSMI is defined as a supremum over pairs of matrices from the Stiefel manifold (cf. \cref{def:msmi_def}). We first show that changing the optimization domain to the space of all $(\rA,\rB)\in\RR^{d_x\times k}\times \RR^{d_y\times k}$ without changing the mSMI value. Fix $(\rA,\rB)\in\RR^{d_x\times k}\times \RR^{d_y\times k}$ and let $\rA=\rU_{\rA}\Sigma_{\rA}\rV_\rA^\intercal$ and $\rB=\rU_{\rB}\Sigma_{\rB}\rV_\rB^\intercal$ be their compact SVDs, i.e., such that $\Sigma_{\rA},\Sigma_{\rB}\in\RR^{k\times k}$. By invariance of mutual information, we have
\[
\sI(\rAt X;\rBt Y)=\sI(\rV_{\rA}\Sigma_{\rA}\rU_\rA^\intercal X;\rV_{\rB}\Sigma_{\rB}\rU_\rB^\intercal Y)=\sI(\rU_\rA^\intercal X;\rU_\rB^\intercal Y),
\]
since $\rV_\rA\Sigma_\rA,\rV_\rB\Sigma_\rB\in\RR^{k\times k}$ are invertible. Noticing that $(\rU_\rA,\rU_\rB)\in\sti(k,d_x)\times\sti(k,d_y)$, we obtain
\[\msmi_k(X;Y)=\sup_{(\rA,\rB)\in\RR^{d_x\times k}\times \RR^{d_y\times k}}\sI(\rAt X;\rBt Y).\]

Next, we show that we may equivalently optimize with the added unit variance constraint. For $(\rA,\rB)\in\RR^{d_x\times k}\times \RR^{d_y\times k}$, define $\Gamma_\rA=\rAt \Sigma_X\rA$ and $\Gamma_\rB=\rBt \Sigma_Y\rB$, and consider their respective eigenvalue decompositions $\Gamma_\rA=\rW_\rA\Lambda_\rA\rW_\rA^\intercal$ and $\Gamma_\rB=\rW_\rB\Lambda_\rB\rW_\rB^\intercal$. By invariance, once more, we have
\[\sI(\rAt X;\rBt Y)=\sI\big(\Lambda_\rA^{-\frac 12}\rU_\rA^\intercal \rAt X;\Lambda_\rB^{-\frac 12}\rU_\rB^\intercal \rBt Y\big)=\sI(\tilde\rA^\intercal X;\tilde\rB^\intercal Y\big),
\]
where $\tilde{\rA}=\rA\rU_\rA\Lambda_\rA^{-1/2}$ and $\tilde{\rB}=\rB\rU_\rB\Lambda_\rB^{-1/2}$, for which we have $\tilde\rA^\intercal\Sigma_X\tilde\rA=\tilde\rB^\intercal\Sigma_Y\tilde\rB=\rI_k$. This proves the first equality in \cref{lemma:Gaussian_msmi_domain}.

For the second inequality, fix $(\rA,\rB)\in\RR^{d_x\times k}\times\RR^{d_y\times k}$ with $\rA^\intercal\Sigma_X\rA=\rB^\intercal\Sigma_Y\rB=\rI_k$, and note that $\rAt X\sim \cN(0,\rAt\Sigma_X\rA)$ and $\rBt Y\sim \cN(0,\rBt\Sigma_Y\rB)$ are jointly Gaussian with cross-covariance $\rAt\Sigma_{XY}\rB$. By the closed-form expression for mutual information between Gaussians (cf. \cite[Example 3.4]{polyanskiy2022ITbook}), we have
\begin{align*}
\sI(\rAt X;\rBt Y)&=-\frac{1}{2}\log\det\left(\left[\begin{array}{cc}
    \rAt\Sigma_X\rA & \rAt\Sigma_{XY}\rB \\
    \rBt\Sigma_{XY}^\intercal\rA & \rBt\Sigma_Y\rB
\end{array}\right]\right)\\
&=-\frac 12\log\big(\rI_k-(\rAt\Sigma_{XY}\rB)^\intercal(\rAt\Sigma_{XY}\rB)\big),
\end{align*}
where the last equality uses the unit variance property and Schur's determinant formula.
\end{proof}

Armed with \cref{lemma:Gaussian_msmi_domain}, we are in place to prove \cref{prop:k_msmi_cca}. Since the CCA solutions $(\rA_\cca,\rB_\cca)$ satisfy the unit variance constraint, we trivially have $\msmi_k(X;Y)\geq \sI(\rA_\cca^\intercal X;\rB_\cca^\intercal Y)$. Recall that $(\rA_\cca,\rB_\cca)=(\Sigma_X^{-1/2}\rU,\Sigma_Y^{-1/2}\rV)$, where $\rU$ and $\rV$ are obtained from the SVD of $\rT_{XY}=\Sigma_X^{-1/2}\Sigma_{XY}\Sigma_Y^{-1/2}=\rU\Lambda\rV^\intercal$ and contain its first $k$ left- and right-singular vectors of $\rT_{XY}$ in their columns; the matrix $\Lambda$ is diagonal and contains the top $k$ singular values of $\rT_{XY}$. Noticing that $(\rA_\cca^\intercal\Sigma_{XY}\rB_\cca)^\intercal(\rA_\cca^\intercal\Sigma_{XY}\rB_\cca)=\Lambda^2$, we have
\begin{equation}
\msmi_k(X;Y)\geq \sI(\rA_\cca^\intercal X;\rB_\cca^\intercal Y)=-\frac 12\log\det(\rI_k-\Lambda^2)=-\frac 12\sum_{i=1}^k\log\big(1-\sigma_i(\rT_{XY})^2\big).\label{eq:Gaussian_msmi_LB}    
\end{equation}
Further observe that $\sigma_i(\rT_{XY})\leq 1$, for all $i=1,\ldots k$. Indeed, for any unit vectors $(a,b)\in\unitsphx\times\unitsphy$, the value $a^\intercal \rT_{XY}b=a^\intercal\Sigma_X^{-1/2}\Sigma_{XY}\Sigma_Y^{-1/2}b$ is exactly the correlation~coefficient $\rho(a^\intercal X,b^\intercal Y)\in[-1,1]$. Taking the supremum over all such vector pairs we arrive at the operator norm of $\rT_{XY}$, which coincides with its largest singular value. In sum, $\sigma_1(\rT_{XY})=\|\rT_{XY}\|_{\op}\leq 1$

For the opposite inequality, we use a generalization of the Poincar\'e separation theorem from \cite[Theorem 2.2]{rao1979separation}, which is restated next for completeness.

\begin{theorem}[Generalized Poincar\'e separation \cite{rao1979separation}]\label{thm:poincare_separation}
 Let $\Sigma\in\RR^{m\times n}$ and $(\rA,\rB)\in\sti(r,m)\times\sti(k,n)$. Then   
\[
\sigma_{t+i}(\Sigma)\leq \sigma_i(\rAt\Sigma\rB)\leq \sigma_i(\Sigma),\quad i=1,\ldots,r\wedge k,
\]
where $t=m+n-r-k$.
\end{theorem}
For any $(\rA,\rB)\in\RR^{d_x\times k}\times \RR^{d_y\times k}$ with $\rAt \Sigma_X\rA=\rBt \Sigma_Y\rB=\rI_k$, defining $\rA_X=\Sigma_X\rA$ and $\rB_Y=\Sigma_Y\rB$, note that $(\rA_X,\rB_Y)\in\sti(k,d_x)\times \sti(k,d_y)$ and $\rAt \Sigma_{XY}\rB=\rA_X^\intercal\rT_{XY}\rB_Y$.
By \cref{thm:poincare_separation}, we obtain
\begin{equation}
\sigma_i(\rAt\Sigma_{XY}\rB)=\sigma_i(\rA_X^\intercal\rT_{XY}\rB_Y)\leq \sigma_i(\rT_{XY}),\quad i=1,\ldots,k.\label{eq:Gaussian_msmi_singular_vals}
\end{equation}
Starting from the log-determinant expression in \cref{lemma:Gaussian_msmi_domain}, consider
\begin{align*}
    \msmi_k(X;Y)&=\sup_{\substack{(\rA,\rB)\in \RR^{d_x\times k}\times\RR^{d_y\times k}:\\ \rA^\intercal \Sigma_X\rA=\rB^\intercal\Sigma_Y\rB=\rI_k}}-\frac{1}{2}\log\det\big(\rI_k-(\rA^\intercal\Sigma_{XY}\rB)^\intercal(\rA^\intercal\Sigma_{XY}\rB)\big)\\
    &=\sup_{\substack{(\rA,\rB)\in \RR^{d_x\times k}\times\RR^{d_y\times k}:\\ \rA^\intercal \Sigma_X\rA=\rB^\intercal\Sigma_Y\rB=\rI_k}}-\frac{1}{2}\sum_{i=1}^k\log\big(1-\sigma_i(\rAt\Sigma_{XY}\rB)^2\big)\\
    &=\sup_{\substack{(\rA,\rB)\in \RR^{d_x\times k}\times\RR^{d_y\times k}:\\ \rA^\intercal \Sigma_X\rA=\rB^\intercal\Sigma_Y\rB=\rI_k}}-\frac{1}{2}\sum_{i=1}^k\log\big(1-\sigma_i(\rA_X^\intercal\rT_{XY}\rB_Y)^2\big)\\
    &\leq -\frac{1}{2}\sum_{i=1}^k\log\big(1-\sigma_i(\rT_{XY})^2\big),\numberthis\label{eq:Gaussian_msmi_UB}
\end{align*}
where the last two steps use \eqref{eq:Gaussian_msmi_singular_vals} and the fact that $x\mapsto -\log(1-x)$ is monotonically increasing (for the last inequality). Combining \eqref{eq:Gaussian_msmi_LB} and \eqref{eq:Gaussian_msmi_UB} yields the result.$\hfill\square$


\subsection{Equivalence Between Max-Sliced Entropy and PCA}\label{appendix:gaussian_pca}

The argument is similar to that in the proof of \cref{prop:k_msmi_cca}. Let $X\sim \cN(m,\Sigma)$ and assume w.l.o.g. that $m=0$ and $\Sigma\in\RR^{d\times d}$ is full-rank. The $k$-dimensional PCA problem for $\Sigma$ is
\[
\sup_{\rA\in\sti(k,d)} \Tr(\rAt\Sigma\rA)
\]
and the global optimum $\rA_{\mathsf{PCA}}$ is the matrix that contains the first $k$ eigenvectors of $\Sigma$ (i.e., corresponding to the largest $k$ eigenvalues). Consequently, 
\begin{align*}
\msh_k(X) &= \sup_{\rA\in\sti(k,d)}\sh(A^\intercal X)\\
&= \sup_{\rA\in\sti(k,d)}\frac 12\log\big((2\pi e)^k \det(\rA^\intercal \Sigma \rA)\big)\\
&= \sup_{\rA\in\sti(k,d)}\frac 12\sum_{i=1}^k\log\big(2\pi e \lambda_i(\rAt\Sigma\rA)\big)\\
&=\frac 12\sum_{i=1}^k\log\big(2\pi e \lambda_i(\Sigma)\big),
\end{align*}
where the second equality is the formula for the differential entropy of a $k$-dimensional Gaussian random vector, while the last one is justified via two-sided inequalities as follows. The $\geq$ relation follows by substituting the PCA solution $\rA_\mathsf{PCA}$. For the reverse inequality we use the Poincar\'e separation theorem (cf., e.g., \cite[Theorem~10.10]{magnus2019matrix}), whereby for any $\rA\in\sti(k,d)$ we have
\[
\lambda_{d-k-i}(\Sigma)\leq \lambda_i(\rAt\Sigma\rA)\leq \lambda_i(\Sigma),\quad \forall i=1,\ldots,k.
\]
Together with the monotonicity of the logarithm this yields the result. 

\subsection{Proof of Theorem \ref{thm:msmi_NE}}\label{appendix:msmi_NE}
The proof leverages an error bound that is uniform over all $\mu_{XY}\in\cP_k(M,b)$, from which a minimax bound will follow.
Fix $\mu_{XY}\in\cP_k(M,b)$. We have
$$
    \EE\left[ \left|\msmi_k(X,Y) - \widehat{\mathsf{SI}}_k^{n,l}\right| \right]\leq\max_{(\rA,\rB)\in\sti(k,d_x)\times\sti(k,d_y)}\EE\left[ \left|\sI(\rA^\intercal X,\rB^\intercal Y) - \widehat{\mathsf{I}}((\rA^\intercal X)^n,(\rB^\intercal Y)^n)\right| \right],
$$
where $\widehat{\mathsf{I}}((\rA^\intercal X),(\rB^\intercal Y))$ is a neural estimator of $\sI(\rA^\intercal X,\rB^\intercal Y)$, calculated from $((\rA^\intercal X)^n,(\rB^\intercal Y)^n):=\{(\rA^\intercal X_i,\rB^\intercal Y_i)\}_{i=1}^n$ with $(X^n,Y^n)$ that are independent and identically distributed according to $\mu_{XY}$. 
Thus, a uniform bound over $(\rA,\rB)\in\sti(k,d_x)\times\sti(k,d_y)$ will suffice to bound the maximum.
We obtain such bound via the following result \cite[Lemma~5]{goldfeld2022k}:
\begin{proposition}[Neural estimation of $\sI(\rA^\intercal X; \rB^\intercal Y)$]\label{prop:unif_ne_bound}
    Let $\mu_{XY}\in\cP_k(M,b)$. Then, uniformly in $(\rA,\rB)\in\sti(k,x_d)\times\sti(k,d_y)$, we have the neural estimation bound
    \begin{equation}\label{eq:moment_constraints}
        \EE\left[ \left|\sI(\rA^\intercal X,\rB^\intercal Y) - \widehat{\mathsf{I}}((\rA^\intercal X)^n,(\rB^\intercal Y)^n)\right| \right]\leq Ck^{\frac{1}{2}}(l^{-\frac{1}{2}}+kn^{-\frac{1}{2}})
    \end{equation}
    where the constant $C$ depends on $M,b,k$, and $\|\cX\times\cY\|$
\end{proposition}
The proof of proposition \ref{prop:unif_ne_bound} consists of showing that the densities of the pushforward measures $(\mathfrak{p}^\rA,\mathfrak{p}^\rB)_\sharp\mu_{XY}$ and $\mathfrak{p}^\rA_\sharp\mu_X\otimes\mathfrak{p}^\rB_\sharp\mu_Y$ satisfy certain smoothness conditions that are sufficient for the spectral condition of the neural estimation bound from \cite{sreekumar2022neural}.

Consequently, we have a bound that is uniform in both $(\rA,\rB)\in\sti(k,d_x)\times\sti(k,d_y)$ and $\mu_{XY}\in\cP_k(M,b)$, providing us with the desired result.$\hfill\square$

\section{Maximization of Max-Sliced Entropy}\label{supp:msh_maximization}

Let $\mu\in\cP(\cX)$ and let $\mu_\rA=\mathfrak{p}^\rA_\sharp\mu$ be the distribution of $\rAt X$.
We next define two classes of distributions and characterize the corresponding max-sliced entropy maximizing distribution from each class.

\paragraph{Mean and covariance constraints.}
The following lemma shows that the Gaussian distribution maximizes max-sliced entropy under first and second moment constraints.
\begin{lemma}
    Let $\cP_1(m,\Sigma):=\big\{\mu\in\cP(\RR^d):\,\supp(\mu)=\RR^d\,,\,\EE_\mu[X]=m\,,\,\EE_\mu\big[(X-m)(X-m)^\intercal \big]=\Sigma\big\}$ be the class of probability measures on $\RR^d$ with fixed mean and covariance. Then,
    \begin{equation}
        \argmax_{\mu\in\cP_1(m,\Sigma)}\msh_k(\mu) = \cN(m,\Sigma).
    \end{equation}
\end{lemma}
\begin{proof}
    Fix $\mu\in\cP_1(m,\Sigma)$ and $\rA\in\sti(k,d)$. The distribution of $\rA^\intercal X$ also has fixed mean $\rA^\intercal m$ and covariance matrix $\rA^\intercal\Sigma\rA$.
    Among all distributions with these mean and covariance, it is the Gaussian distribution $\cN(\rA^\intercal m,\rA^\intercal\Sigma\rA)$ that maximizes differential entropy \cite{CovThom06}.
    Consequently, 
    \[\sup_{\mu\in\cP_1(m,\Sigma)}\mspace{-5mu}\msh_k(\mu)\mspace{-2mu}=\mspace{-5mu}\sup_{\rA\in\sti(k,d)}\sup_{\mu\in\cP_1(m,\Sigma)}\mspace{-5mu}\sh(\mathfrak{p}^\rA_\sharp\mu)\leq\mspace{-5mu}\sup_{\rA\in\sti(k,d)}\sh\big(\cN(\rA^\intercal m,\rA^\intercal\Sigma\rA)\big)=\msh_k\big(\cN(m,\Sigma)\big),
    \]
    and the inequality is achieved by setting $\mu=\cN(m,\Sigma)$. This proves the claim. 
\end{proof}

\paragraph{Support inside $d$-dimensional ball.}
The next claim is analogous to the fact that the uniform distribution maximizes differential entropy over the class of compactly supported distributions. 
\begin{lemma}\label{lemma:max_msh_unif}
    Let $\cP_2(c,r):=\big\{ \mu\in\cP(\RR^d):\,\supp(\mu)\subseteq\BB_d(c,r) \big\}$ be the class of probability measures supported inside a $d$-dimensional ball, centered at $c\in\RR^d$ with radius $r>0$. Then    
    $$\mathsf{Unif}\big(\BB_k((c_1,\ldots,c_k),r))\big)\otimes \delta_{(c_{k+1},\ldots,c_d)}\in \argmax_{\mu\in\cP_2(c,r)}\msh_k(\mu),
    $$
    i.e., the max-sliced entropy maximizing distribution is the uniform distribution on a $k$-dimensional ball, with the remaining $n-k$ variables equal to the corresponding entries of $c$. 
    The corresponding maximal max-sliced entropy is
    \[\sup_{\mu\in\cP_2(c,r)}\msh_k(\mu)=\log\big((\pi r^2)^{k/2}/\Gamma(k/2+1)\big),
    \] where $\Gamma$ is the Gamma function.
\end{lemma}
\begin{proof}


    The proof shows that within the projected $k$-dimensional space the entropy maximizing distribution is the uniform distribution over the $k$-dimensional ball of radius $r$.

    First, due to the maximum entropy principle \cite[Theorem~12.1.1]{CovThom06}, when the only constraint on the distribution family is a compact support $\cX$, maximum entropy is achieved by the uniform distribution with density $p_X=\exp(-\log\mathsf{Vol}(\cX))$.
    We know that every linear $k$-dimensional projection of a $d$-dimensional ball is a $k$-dimensional ball, implying that the support set of each $k$-dimensional projection is compact. Thus, we look for a distribution $\mu\in\cP_2(c,r)$ with a $k$-dimensional projection that (i) has the largest possible volume of support in the projected space, and (ii) is uniform distribution over this projected support.
    Such a distribution, if it exists, will be the maximizer of $\msh_k(\mu)$ over $\mu\in\cP_2(c,r)$. 

    The solution to point (i)) above is simple: the largest possible projected support set is the $k$-dimensional ball of radius $r$. It remains to find a distribution with a $k$-dimensional projection that is uniform over this support set. This is achieved by 
    \[
    \mu_k:=\mathsf{Unif}\big(\BB_k((c_1,\ldots,c_k),r)\big)\otimes \delta_{(c_{k+1},\ldots,c_d)},
    \]
    noting that the projection $\rA_k:=[\rI_k;0_{k\times d}]^\intercal$, where $0_{k\times d}$ is a matrix with zero entries, yields
    $$
    \mathfrak{p}^{\rA_k}_\sharp\mu_k=\mathsf{Unif}(\BB_k((c_1,\ldots,c_k),r)).
    $$
    

    
    The maximum max-sliced entropy is thus given by the logarithm of the ball volume, i.e.,
    $$
    \max_{\mu\in\cP_2(0,r)}\msh_k(\mu)=\log\big((\pi r^2)^{k/2}/\Gamma(k/2+1)\big)
    $$

    Note that the proposed solution holds for any rotation of $\mu_k$ as follows.
    Let $\rU\in\sti(d,d)$ be orthogonal and denote $\mu_{k,\rU} = \mathfrak{p}^{\rU}_\sharp\mu_k$, which is the law of $\rU^\intercal X$, for $X\sim \mu$. The entropy maximizing distribution can be obtained with a respective rotation of $\rA_k$.
    Take $\rA_{k,\rU} = \rU\rA_k$, we have
    $$
    \rA_{k,\rU}^\intercal \rU^\intercal X = \rA_k^\intercal X,
    $$
    as desired.
    This, in turn, implies that $\argmax_{\mu\in\cP_2(c,r)}\msh_k(\mu)$ is not unique.
\end{proof}

\section{Additional Implementation Details}\label{appendix:imeplementation}

\begin{wrapfigure}{R}{0.45\textwidth}
\vspace{-0.4cm}
    \centering
    \includegraphics[trim={195pt 490pt 190pt 80pt}, clip,scale=0.65]{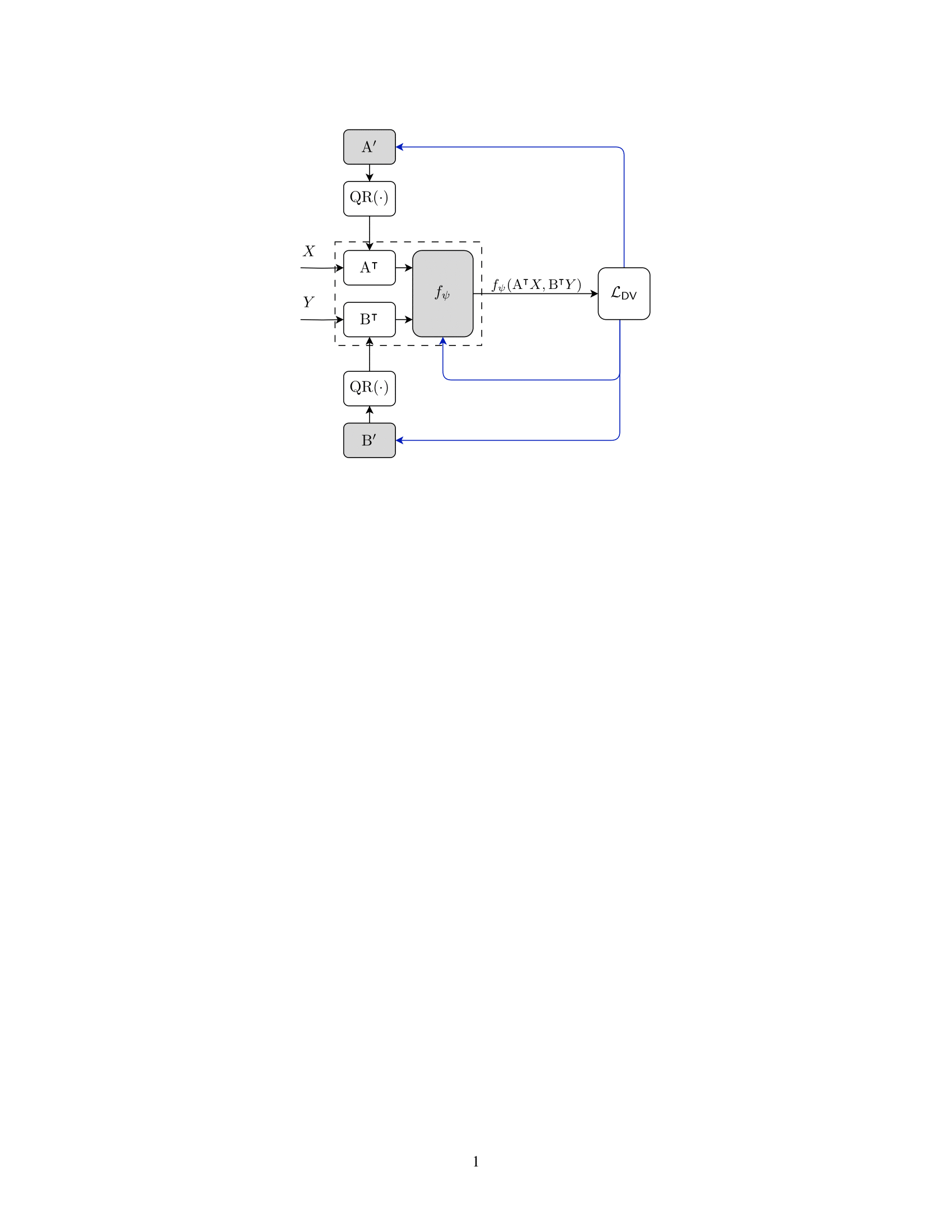}
    \caption{Neural estimation of mSMI. $\mathrm{QR}(\cdot)$ blocks denote the application of a QR decomposition, from which we take the Q (orthogonal) part. Blue lines denote gradient propagation and shaded blocks denote parametric models.}
    \label{fig:msmi_ne}
    \vspace{-0.85cm}
\end{wrapfigure}

\paragraph{Neural estimation.}
We consider the popular seperable critic \cite{song2019understanding,tschannen2019mutual,guo2022tight}, which is given by $g(x,y)=h_1(x)^\intercal h_2(y)$, such that $h_1$ and $h_2$ are two independent copies of the same MLP architecture with embedding dimension $d_o$. 
In our setting the MLP architecture is given by two hidden layers with an exponential linear unit activation \cite{clevert2015fast} whose hidden dimension of 256. The MLP output dimension is 32.
We utilize the Adam optimizer \cite{kingma2014adam} with initial learning rate of $2\times 10^{-4}$.
both the mSMI and the aSMI estimators instances are implemented with similar copies of the aforementioned critic.
The aSMI is estimated via the parallel estimator from \cite{goldfeld2022k}, following their choice of $m=1000$.

\paragraph{Algorithm.} We employ a minibatch stochastic gradient-ascent scheme. The DV potential network $f_\psi$ and slicing directions $(\rA',\rB')$ are randomly initialized. Each iteration begins by sampling a batch of positive and negative samples, which are then projected via $(\rA,\rB)$ where $(\rA, \rB)$ are the projections of $(\rA',\rB')$ onto the Stiefel manifold, as computed by QR decomposition. The projected samples are passed through $f_\psi$ and the objective of \eqref{eq:mSMI_NE} is calculated. Finally, we update the parameters $(\rA, \rB, \psi)$ are updated. The mechanism is visualised in Figure \ref{fig:msmi_ne}.

\paragraph{Independence testing.}
We follow the setting of a latent shared random variable from \cite{goldfeld2022k}, given as follows.
Let $(Z_1,Z_2)\sim\cN(0,\rI_d)$ and $V\sim\cN(0,\rI_{d'})$ be independent. We set $X=\rP_1 V + Z_1$ and $Y=\rP_2 V + Z_2$, where $\rP_1,\rP_2\in\RR^{d\times d'}$ are projection matrices with i.i.d normally distributed entries.
We estimate the aSMI via the parallel aSMI methodology from \cite{goldfeld2022k} with $m=1000$ estimator instances, and the mSMI via the LIPO algorithm  \cite{malherbe2017global} with a stopping criteria after $m=1000$ samples.
The mutual information estimator we use is the Kozachenko-Leonenko estimator \cite{kraskov2004estimating} and the AUC-ROC is computed over 100 trials.

\paragraph{Multi-view representation learning.}
The setup follows the CCA and DCCA implementations from \cite{wang2015deep} and mutual information estimation of \cite{tschannen2019mutual}.
The MLP architectures are similar to the ones used to construct the seperable critic, as described in the neural estimation implementation.
The procedure consists on dividing each image to its top and bottom halves and flattening each half.
These flattened halves are then projected using the corresponding projection models.
The classification is performed via multi-class logistic regression using SAGA \cite{defazio2014saga}, as implemented via \texttt{scikit-learn} python library.

\paragraph{Max-sliced InfoGAN.}
Following the setting of \cite{goldfeld2022k}, we replace the mutual information regularizer from the original InfoGAN implementation\footnote{Code implementation is based on \texttt{\url{https://github.com/Natsu6767/InfoGAN-PyTorch}}} with an mSMI regularizer, and maximize the compound loss over the InfoGAN parameters and the slice directions. We empirically observed that this joint optimization sometimes steers the model away from the true maximizing slice. To address this and improve the accuracy of the overall mSMI estimate, we introduced $m$ independent mSMI networks and independently initialize and optimize them. At each iteration, each sliced model yields a corresponding mSMI estimate based on its current slice. Our final mSMI estimate is then obtained by taking the maximum over the $m$ estimates, which is more likely to be close to the true mSMI. For differentiability, we approximate the maximum with a logsumexp function. Note that the analogous average-sliced InfoGAN experiment from  \cite[Section 5]{goldfeld2022k} considers the average of $m=1000$ random slicing directions, while our max-sliced InfoGAN uses only $m\leq 30$. This again demonstrates the utility of mSMI for learning tasks due to its low computational overhead.

\section{Additional Multi-View Representation Learning Results}\label{appendix:reprensetaion}
Table \ref{table:accuracy_supp} provides results on a wider range of $k$ values.
In Table \ref{table:cifar} we present a comparison of the multi-view setting for the CIFAR10 dataset \cite{krizhevsky2009learning}, which is a benchmark dataset for image classification, consisting of 60,000 small color images divided into 10 classes, similar to the MNIST dataset but with more complex and varied objects.
Because the images in the CIFAR dataset consist of three channels, straightforward flattening and projection cannot be applied.
We therefore consider a simple convolutional NN architecture that consists of two convolutional layer with padding and stride of $2$, followed by a layer normalization, average pool and a fully connected output layer to results with a $k$-dimensional output.
It is clear from Table \ref{table:cifar} that mSMI outperforms the DCCA objective in the CIFAR setting.
The results show that under no fine tuning the convolutional DCCA method doesn't scale well with the projection dimension, with optimal results for $k=30$, while the mSMI methodology continues to improve with $k$.
All results are averaged over 10 different seeds.


\begin{table}[h]
\caption{Full classification Results on MNIST}
\label{table:accuracy_supp}
\centering
{
\begin{tabular}{|c|cc|cc|}
\toprule
\textbf{$k$} & \textbf{Linear CCA} & \textbf{Linear mSMI} & \textbf{MLP DCCA} & \textbf{MLP mSMI} \\
\midrule
1 & 0.261$\pm$0.03 & \textbf{0.274}$\pm$0.02 & 0.284$\pm$0.03 & \textbf{0.291}$\pm$0.02 \\
2 & 0.32$\pm$0.02 & \textbf{0.346}$\pm$0.02 & 0.314$\pm$0.03 & \textbf{0.417}$\pm$0.02 \\
4 & 0.42$\pm$0.01 & \textbf{0.478}$\pm$0.02 & 0.441$\pm$0.04 & \textbf{0.546}$\pm$0.01 \\
6 & 0.502$\pm$0.01 & \textbf{0.634}$\pm$0.01 & 0.599$\pm$0.01 & \textbf{0.655}$\pm$0.01 \\
8 & 0.553$\pm$0.03 & \textbf{0.666}$\pm$0.01 & 0.645$\pm$0.02 & \textbf{0.665}$\pm$0.01 \\
10 & 0.595$\pm$0.01 & \textbf{0.702}$\pm$0.01 & 0.668$\pm$0.01 & \textbf{0.715}$\pm$0.01 \\
12 & 0.614$\pm$0.02 & \textbf{0.751}$\pm$0.01 & 0.697$\pm$0.01 & \textbf{0.753}$\pm$0.01 \\
14 & 0.65$\pm$0.01 & \textbf{0.767}$\pm$0.01 & 0.71$\pm$0.01 & \textbf{0.767}$\pm$0.01\\
16 & 0.673$\pm$0.02 & \textbf{0.775}$\pm$0.01 & 0.730$\pm$0.02 & \textbf{0.779}$\pm$0.01 \\
18 & 0.689$\pm$0.01 & \textbf{0.785}$\pm$0.006 & 0.762$\pm$0.009 & \textbf{0.779}$\pm$0.01 \\
20 & 0.704$\pm$0.007 & \textbf{0.79}$\pm$0.006 & 0.774$\pm$0.01 & \textbf{0.798}$\pm$0.01 \\
\bottomrule
\end{tabular}
}
\end{table}

\begin{table}[ht]
\caption{Result in the CIFAR10 dataset.}
\centering
\label{table:cifar}
\begin{tabular}{|c|cc|}
\toprule
\textbf{$k$} & \textbf{DCCA} & \textbf{mSMI}  \\
\midrule
1 & 0.1281 $\pm$ 0.0387 & \textbf{0.1374} $\pm$ 0.0310 \\
5 & 0.1324 $\pm$ 0.0397 &\textbf{ 0.1714 }$\pm$ 0.0110 \\
10 & 0.1802 $\pm$ 0.0050 & \textbf{0.2040} $\pm$ 0.0070 \\
20 & 0.2262 $\pm$ 0.0142 & \textbf{0.2471} $\pm$ 0.0090 \\
30 & 0.2433 $\pm$ 0.0196 & \textbf{0.2487} $\pm$ 0.0105 \\
40 & 0.1999 $\pm$ 0.0397 & \textbf{0.2508} $\pm$ 0.0254 \\
50 & 0.1627 $\pm$ 0.0183 & \textbf{0.2555} $\pm$ 0.0168 \\
60 & 0.1840 $\pm$ 0.0196 & \textbf{0.2792} $\pm$ 0.0115 \\
70 & 0.1973 $\pm$ 0.0085 & \textbf{0.2876} $\pm$ 0.0058 \\
\bottomrule
\end{tabular}
\end{table}

\section{Additional Fairness Representation Learning Results}
Table \ref{table:Adult} provides fairness representation learning results on the UCI Adult dataset. This dataset consists of 48,842 rows of US Census data, with 14 features describing educational background, age, race, marital status, and others. Here, the outcome $Y$ is a binary indicator of whether the individual has an income at least US\$50,000, and the sensitive attribute $T$ is race.
\label{app:adult}
\begin{table}[ht]
 \caption{Learning a fair representation of the Adult dataset, following the setup of \cite{chen2023scalable}.  }
\label{table:Adult}
\centering
\resizebox{.9\linewidth}{!}{
\begin{tabular}{|c|c|c|c|c|c|c|c|c|c|}
\hline
 & \textbf{N/A} & \textbf{Slice \cite{chen2023scalable}} & \multicolumn{7}{c|}{\textbf{mSMI (ours)}}  \\
 \hline
 \cellcolor{gray}& \cellcolor{gray}& \cellcolor{gray} &$k=1$ & $k=2$ & $k=3$ & $k=4$ & $k=5$ & ${k=6}$ & $k=7$\\
\hline
$\rho_{\mathsf{HGR}}^\ast(Z,Y) \uparrow$ & 0.998 & 0.979 & 0.998 & 0.972 & 0.947 & \textbf{0.992} & 0.971 & 0.991 & 0.962 \\
\hline
$\rho_{\mathsf{HGR}}^\ast(Z,A) \downarrow$ & 0.990 & 0.068 &0.43 & 0.393 & 0.137 & \textbf{0.052} & 0.053 &0.137 & 0.74 \\
\hline
\end{tabular}}
\end{table}



\end{document}

%% file: Figures/NE_k1/NE_k1_asi_rho05.tex
\begin{tikzpicture}

\definecolor{chocolate2267451}{RGB}{226,74,51}
\definecolor{dimgray85}{RGB}{85,85,85}
\definecolor{gainsboro229}{RGB}{229,229,229}

\begin{axis}[
clip mode=individual,
width=9cm, height=3.5cm,
axis background/.style={fill=gainsboro229},
axis line style={white},
log basis x={10},
tick align=outside,
tick pos=left,
x grid style={white},
xmajorgrids,
xmin=397.164117362141, xmax=62946.2705897083,
xmode=log,
xtick style={color=dimgray85},
xtick={10,100,1000,10000,100000,1000000},
xticklabels={
  \(\displaystyle {10^{1}}\),
  \(\displaystyle {10^{2}}\),
  \(\displaystyle {10^{3}}\),
  \(\displaystyle {10^{4}}\),
  \(\displaystyle {10^{5}}\),
  \(\displaystyle {10^{6}}\)
},
y grid style={white},
ylabel=\textcolor{dimgray85}{\large $\asmi$ [nats]},
ymajorgrids,
ymin=0.021, ymax=0.025448349136428,
ytick style={color=dimgray85},
]
\addplot [semithick, chocolate2267451, mark=*, mark size=3, mark options={solid}]
table {%
50000 0.0215479441488876
10000 0.0217599544766737
5000 0.0219538848063418
1000 0.0233456329460136
500 0.0252626155655927
};
\addplot [semithick, chocolate2267451, dashed]
table {%
50000 0.0215479441488876
10000 0.0215479441488876
5000 0.0215479441488876
1000 0.0215479441488876
500 0.0215479441488876
100 0.0215479441488876
};
\pgfresetboundingbox
\end{axis}
\end{tikzpicture}

%% file: Figures/NE_k1/NE_k1_maxgrad_rho0.5.tex
\begin{tikzpicture}

\definecolor{chocolate2267451}{RGB}{52,138,189}
\definecolor{dimgray85}{RGB}{85,85,85}
\definecolor{gainsboro229}{RGB}{229,229,229}

\begin{axis}[
width=9cm, height=3.5cm,
axis background/.style={fill=gainsboro229},
axis line style={white},
log basis x={10},
tick align=outside,
tick pos=left,
x grid style={white},
xlabel=\textcolor{dimgray85}{},
xmajorgrids,
xmin=397.164117362141, xmax=62946.2705897083,
xmode=log,
xtick style={color=dimgray85},
xtick={10,100,1000,10000,100000,1000000},
xticklabels={
  \(\displaystyle {10^{1}}\),
  \(\displaystyle {10^{2}}\),
  \(\displaystyle {10^{3}}\),
  \(\displaystyle {10^{4}}\),
  \(\displaystyle {10^{5}}\),
  \(\displaystyle {10^{6}}\)
},
y grid style={white},
ylabel=\textcolor{dimgray85}{\large $\msmi$ [nats]},
ymajorgrids,
ymin=0.137318582087755, ymax=0.36,
ytick style={color=dimgray85}
]
\addplot [semithick, chocolate2267451, mark=*, mark size=3, mark options={solid}]
table {%
50000 0.34142941236496
10000 0.342191725969315
5000 0.340315163135529
1000 0.274316489696503
500 0.147074446082115
};
\addplot [semithick, chocolate2267451, dashed]
table {%
50000 0.34142941236496
10000 0.34142941236496
5000 0.34142941236496
1000 0.34142941236496
500 0.34142941236496
100 0.34142941236496
};
\end{axis}

\end{tikzpicture}

%% file: Figures/NE_k1/NE_k1_time_log_rho0.5.tex
\begin{tikzpicture}

\definecolor{chocolate2267451}{RGB}{226,74,51}
\definecolor{dimgray85}{RGB}{85,85,85}
\definecolor{gainsboro229}{RGB}{229,229,229}
\definecolor{steelblue52138189}{RGB}{52,138,189}

\begin{axis}[
width=9cm, height=3.5cm,
legend style={at={(1,0)}, anchor=south east , fill opacity=0.65, draw opacity=1, text opacity=1, draw=none},
axis background/.style={fill=gainsboro229},
axis line style={white},
log basis x={10},
log basis y={10},
tick align=outside,
tick pos=left,
x grid style={white},
xlabel=\textcolor{dimgray85}{\large  $n$},
xmajorgrids,
xmin=397.164117362141, xmax=62946.2705897083,
xmode=log,
xtick style={color=dimgray85},
xtick={10,100,1000,10000,100000,1000000},
xticklabels={
  \(\displaystyle {10^{1}}\),
  \(\displaystyle {10^{2}}\),
  \(\displaystyle {10^{3}}\),
  \(\displaystyle {10^{4}}\),
  \(\displaystyle {10^{5}}\),
  \(\displaystyle {10^{6}}\)
},
y grid style={white},
ylabel=\textcolor{dimgray85}{\large Time [sec.]},
ymajorgrids,
ymin=0.0261518138057445, ymax=502.54122536662,
ymode=log,
ytick style={color=dimgray85},
ytick={0.001,0.01,0.1,1,10,100,1000,10000},
yticklabels={
  \(\displaystyle {10^{-3}}\),
  \(\displaystyle {10^{-2}}\),
  \(\displaystyle {10^{-1}}\),
  \(\displaystyle {10^{0}}\),
  \(\displaystyle {10^{1}}\),
  \(\displaystyle {10^{2}}\),
  \(\displaystyle {10^{3}}\),
  \(\displaystyle {10^{4}}\)
}
]
\addplot [semithick, chocolate2267451, mark=*, mark size=3, mark options={solid}]
table {%
50000 320.96629357338
10000 65.7234039306641
5000 30.6288013458252
1000 6.15503668785095
500 3.27333235740662
};
\addplot [semithick, steelblue52138189, mark=*, mark size=3, mark options={solid}]
table {%
50000 3.62920713424683
10000 0.765299081802368
5000 0.369369268417358
1000 0.0784246921539306
500 0.0409462451934814
};
\addlegendentry{$\asmi$}
\addlegendentry{$\msmi$}
\end{axis}

\end{tikzpicture}

%% file: Figures/independence_comparison_d10_dlatent_3/indep_k1.tex
\begin{tikzpicture}

\definecolor{chocolate2267451}{RGB}{226,74,51}
\definecolor{dimgray85}{RGB}{85,85,85}
\definecolor{gainsboro229}{RGB}{229,229,229}
\definecolor{steelblue52138189}{RGB}{52,138,189}

\begin{axis}[
xlabel={$n$}, 
    xlabel style={at={(axis description cs:2.55,0)},anchor=north}, 
xticklabel style={font=\scriptsize},
yticklabel style={font=\footnotesize},
xscale=0.3, yscale=0.5,
axis background/.style={fill=gainsboro229},
axis line style={white},
log basis x={10},
tick align=outside,
tick pos=left,
x grid style={white},
xmajorgrids,
xmin=7.94328234724282, xmax=1258.92541179417,
xmode=log,
xtick style={color=dimgray85},
xtick={0.1,1,10,100,1000,10000,100000},
xticklabels={
  \(\displaystyle {10^{-1}}\),
  \(\displaystyle {10^{0}}\),
  \(\displaystyle {10^{1}}\),
  \(\displaystyle {10^{2}}\),
  \(\displaystyle {10^{3}}\),
  \(\displaystyle {10^{4}}\),
  \(\displaystyle {10^{5}}\)
},
y grid style={white},
ymajorgrids,
ymin=0.475, ymax=1.025,
ytick style={color=dimgray85}
]
\addplot [thick, steelblue52138189]
table {
    10 0.54
    31 0.55
    100 0.56
    317 0.59
    1000 0.99
};
\addplot [thick, steelblue52138189, dashed]
table {
    10 0.6799666666666667
    31 0.8417
    100 0.8975666666666667
    317 0.9737666666666667
    1000 0.9967
};
\node[fill=white, fill opacity=0.5, text opacity=1, anchor=north west, inner sep = 3pt] at (rel axis cs: 0.05, 0.95) {$k=1$};

\end{axis}

\end{tikzpicture}

%% file: Figures/independence_comparison_d10_dlatent_3/indep_k2.tex
\begin{tikzpicture}

\definecolor{chocolate2267451}{RGB}{226,74,51}
\definecolor{dimgray85}{RGB}{85,85,85}
\definecolor{gainsboro229}{RGB}{229,229,229}
\definecolor{steelblue52138189}{RGB}{52,138,189}

\begin{axis}[
    xlabel={$n$}, 
    xlabel style={at={(axis description cs:2.55,-0)},anchor=north}, 
    xticklabel style={font=\scriptsize},
    xscale=0.3, yscale=0.5,
    axis background/.style={fill=gainsboro229},
    axis line style={white},
    log basis x={10},
    tick align=outside,
    tick pos=left,
    x grid style={white},
    xmajorgrids,
    xmin=7.94328234724282, xmax=1258.92541179417,
    xmode=log,
    xtick style={color=dimgray85},
    xtick={0.1,1,10,100,1000,10000,100000},
    xticklabels={
      \(\displaystyle {10^{-1}}\),
      \(\displaystyle {10^{0}}\),
      \(\displaystyle {10^{1}}\),
      \(\displaystyle {10^{2}}\),
      \(\displaystyle {10^{3}}\),
      \(\displaystyle {10^{4}}\),
      \(\displaystyle {10^{5}}\)
    },
    yticklabels={},
    y grid style={white},
    ymajorgrids,
    ymin=0.475, ymax=1.025,
    ytick style={color=dimgray85}
]
\addplot [thick, steelblue52138189, dashed]
table {
    10 0.61
    31 0.95
    100 0.99
    317 1.00
    1000 1.00
};
\addplot [thick, steelblue52138189]
table {
    10 0.59
    31 0.77
    100 0.84
    317 0.96
    1000 1.00
};
\node[fill=white, fill opacity=0.5, text opacity=1, anchor=north west, inner sep = 3pt] at (rel axis cs: 0.05, 0.95) {$k=2$};
\end{axis}

\end{tikzpicture}

%% file: Figures/independence_comparison_d10_dlatent_3/indep_k3.tex
\begin{tikzpicture}

\definecolor{chocolate2267451}{RGB}{226,74,51}
\definecolor{dimgray85}{RGB}{85,85,85}
\definecolor{gainsboro229}{RGB}{229,229,229}
\definecolor{steelblue52138189}{RGB}{52,138,189}

\begin{axis}[
xlabel={$n$}, 
    xlabel style={at={(axis description cs:2.55,-0)},anchor=north}, 
xticklabel style={font=\scriptsize},
xscale=0.3, yscale=0.5,
axis background/.style={fill=gainsboro229},
axis line style={white},
log basis x={10},
tick align=outside,
tick pos=left,
x grid style={white},
xmajorgrids,
xmin=7.94328234724282, xmax=1258.92541179417,
xmode=log,
xtick style={color=dimgray85},
xtick={0.1,1,10,100,1000,10000,100000},
xticklabels={
  \(\displaystyle {10^{-1}}\),
  \(\displaystyle {10^{0}}\),
  \(\displaystyle {10^{1}}\),
  \(\displaystyle {10^{2}}\),
  \(\displaystyle {10^{3}}\),
  \(\displaystyle {10^{4}}\),
  \(\displaystyle {10^{5}}\)
},
yticklabels={},
y grid style={white},
ymajorgrids,
ymin=0.475, ymax=1.025,
ytick style={color=dimgray85},
]
\addplot [thick, steelblue52138189, dashed]
table {
    10 0.53
    31 0.64
    100 0.95
    317 1.00
    1000 1.00
};
\addplot [thick, steelblue52138189]
table {
    10 0.71
    31 0.78
    100 0.94
    317 1.00
    1000 1.00
};
\node[fill=white, fill opacity=0.5, text opacity=1, anchor=north west, inner sep = 3pt] at (rel axis cs: 0.05, 0.95) {$k=3$};

\end{axis}

\end{tikzpicture}

%% file: Figures/independence_comparison_d20_dlatent6/indep_k3.tex
\begin{tikzpicture}

\definecolor{chocolate2267451}{RGB}{226,74,51}
\definecolor{dimgray85}{RGB}{85,85,85}
\definecolor{gainsboro229}{RGB}{229,229,229}
\definecolor{steelblue52138189}{RGB}{226,74,51}

\begin{axis}[
xlabel={$n$}, 
    xlabel style={at={(axis description cs:2.55,-0)},anchor=north}, 
xticklabel style={font=\scriptsize},
xscale=0.3, yscale=0.5,
axis background/.style={fill=gainsboro229},
axis line style={white},
log basis x={10},
tick align=outside,
tick pos=left,
x grid style={white},
xmajorgrids,
xmin=7.94328234724282, xmax=1258.92541179417,
xmode=log,
xtick style={color=dimgray85},
xtick={0.1,1,10,100,1000,10000,100000},
xticklabels={
  \(\displaystyle {10^{-1}}\),
  \(\displaystyle {10^{0}}\),
  \(\displaystyle {10^{1}}\),
  \(\displaystyle {10^{2}}\),
  \(\displaystyle {10^{3}}\),
  \(\displaystyle {10^{4}}\),
  \(\displaystyle {10^{5}}\)
},
yticklabels={},
y grid style={white},
ymajorgrids,
ymin=0.475, ymax=1.025,
ytick style={color=dimgray85}
]
\addplot [thick, steelblue52138189,dashed]
table {%
10 0.5
32 0.5
100 0.93
317 1
1000 1
};
\addplot [thick, steelblue52138189]
table {%
10 0.5932
32 0.7592
100 0.8876
317 0.9876
1000 1
};
\node[fill=white, fill opacity=0.5, text opacity=1, anchor=north west, inner sep = 3pt] at (rel axis cs: 0.05, 0.95) {$k=3$};

\end{axis}

\end{tikzpicture}

%% file: Figures/independence_comparison_d20_dlatent6/indep_k4.tex
\begin{tikzpicture}

\definecolor{chocolate2267451}{RGB}{226,74,51}
\definecolor{dimgray85}{RGB}{85,85,85}
\definecolor{gainsboro229}{RGB}{229,229,229}
\definecolor{steelblue52138189}{RGB}{226,74,51}

\begin{axis}[
xlabel={$n$}, 
    xlabel style={at={(axis description cs:2.55,-0)},anchor=north}, 
xticklabel style={font=\scriptsize},
xscale=0.3, yscale=0.5,
axis background/.style={fill=gainsboro229},
axis line style={white},
log basis x={10},
tick align=outside,
tick pos=left,
x grid style={white},
xmajorgrids,
xmin=7.94328234724282, xmax=1258.92541179417,
xmode=log,
xtick style={color=dimgray85},
xtick={0.1,1,10,100,1000,10000,100000},
xticklabels={
  \(\displaystyle {10^{-1}}\),
  \(\displaystyle {10^{0}}\),
  \(\displaystyle {10^{1}}\),
  \(\displaystyle {10^{2}}\),
  \(\displaystyle {10^{3}}\),
  \(\displaystyle {10^{4}}\),
  \(\displaystyle {10^{5}}\)
},
yticklabels={},
y grid style={white},
ymajorgrids,
ymin=0.475, ymax=1.025,
ytick style={color=dimgray85}
]
\addplot [thick, steelblue52138189,dashed]
table {%
10 0.5
32 0.5
100 0.5
317 0.99
1000 1
};
\addplot [thick, steelblue52138189]
table {%
10 0.7772
32 0.8348
100 0.96
317 1
1000 1
};
\node[fill=white, fill opacity=0.5, text opacity=1, anchor=north west, inner sep = 3pt] at (rel axis cs: 0.05, 0.95) {$k=4$};

\end{axis}

\end{tikzpicture}

%% file: Figures/independence_comparison_d20_dlatent6/indep_k5.tex
\begin{tikzpicture}

\definecolor{chocolate2267451}{RGB}{226,74,51}
\definecolor{dimgray85}{RGB}{85,85,85}
\definecolor{gainsboro229}{RGB}{229,229,229}
\definecolor{steelblue52138189}{RGB}{226,74,51}

\begin{axis}[
xlabel={$n$}, 
    xlabel style={at={(axis description cs:2.55,-0)},anchor=north}, 
xticklabel style={font=\scriptsize},
xscale=0.3, yscale=0.5,
axis background/.style={fill=gainsboro229},
axis line style={white},
log basis x={10},
tick align=outside,
tick pos=left,
x grid style={white},
xmajorgrids,
xmin=7.94328234724282, xmax=1258.92541179417,
xmode=log,
xtick style={color=dimgray85},
xtick={0.1,1,10,100,1000,10000,100000},
xticklabels={
  \(\displaystyle {10^{-1}}\),
  \(\displaystyle {10^{0}}\),
  \(\displaystyle {10^{1}}\),
  \(\displaystyle {10^{2}}\),
  \(\displaystyle {10^{3}}\),
  \(\displaystyle {10^{4}}\),
  \(\displaystyle {10^{5}}\)
},
yticklabels={},
y grid style={white},
ymajorgrids,
ymin=0.475, ymax=1.025,
ytick style={color=dimgray85}
]
\addplot [thick, steelblue52138189,dashed]
table {%
10 0.5
32 0.5
100 0.5
317 0.5
1000 1
};
\addplot [thick, steelblue52138189]
table {%
10 0.7044
32 0.9032
100 0.9988
317 1
1000 1
};
\node[fill=white, fill opacity=0.5, text opacity=1, anchor=north west, inner sep = 3pt] at (rel axis cs: 0.05, 0.95) {$k=5$};

\end{axis}

\end{tikzpicture}